\documentclass{article} 
\usepackage{xy_style}

\usepackage{iclr2024_conference}
\usepackage{times}

\usepackage{amsmath,amsfonts,bm}

\def\eqref#1{equation~\ref{#1}}

\def\1{\bm{1}}

\DeclareMathAlphabet{\mathsfit}{\encodingdefault}{\sfdefault}{m}{sl}
\SetMathAlphabet{\mathsfit}{bold}{\encodingdefault}{\sfdefault}{bx}{n}

\usepackage{hyperref}
\usepackage{url}
\usepackage{subfloat}
\usepackage{subcaption}
\usepackage{wrapfig}

\usepackage{enumitem}
\setitemize{noitemsep,topsep=0pt,parsep=0pt,partopsep=0pt}

\newcommand{\fh}[1]{}
\newcommand{\fhc}[1]{}
\usepackage{soul}
\newcommand{\fhst}[1]{}

\title{Rethinking Adversarial Policies: A Generalized Attack Formulation and Provable Defense in RL}

\author{Xiangyu Liu$^1$, Souradip Chakraborty${^1}$, Yanchao Sun$^2$, Furong Huang$^1$ \\
$^1$University of Maryland, College Park, $^2$J.P. Morgan AI Research\\
\texttt{\{xyliu999\}@umd.edu}
}

\usepackage[toc,page,header]{appendix}
\usepackage{minitoc}

\usepackage{subcaption}

\doparttoc 
\faketableofcontents 

\iclrfinalcopy 

\begin{document}

\maketitle

\begin{abstract}
Most existing works focus on direct perturbations to the victim's state/action or the underlying transition dynamics to demonstrate the vulnerability of reinforcement learning agents to adversarial attacks. 
However, such direct manipulations may not be always realizable.
In this paper, we consider a multi-agent setting where a well-trained victim agent $\nu$ is exploited by an attacker controlling another 
agent $\alpha$ with an \textit{adversarial policy}. Previous models do not account for the possibility that the attacker may only have partial control over 
$\alpha$ or that the attack may produce easily detectable ``abnormal'' behaviors. Furthermore, there is a lack of provably efficient defenses against these adversarial policies. 
To address these limitations, we introduce a generalized attack framework that has the flexibility to model to what extent the adversary is able to control the agent, and allows the attacker to regulate the state distribution shift and produce stealthier adversarial policies. Moreover, we offer a provably efficient defense with polynomial convergence to the most robust victim policy through adversarial training with timescale separation. 
This stands in sharp contrast to supervised learning, where adversarial training typically provides only \textit{empirical} defenses.
Using the Robosumo competition experiments, we show that our generalized attack formulation results in much stealthier adversarial policies when maintaining the same winning rate as baselines. 
Additionally, our adversarial training approach yields stable learning dynamics and less exploitable victim policies.\footnote{Codes are available at \href{https://github.com/xiangyu-liu/Rethinking-Adversarial-Policies-in-RL.git}{https://github.com/xiangyu-liu/Rethinking-Adversarial-Policies-in-RL.git}.}
\end{abstract}

\section{Introduction}
\label{sec:intro}

Despite the huge success of deep reinforcement learning (RL) algorithms across various domains \citep{silver2017mastering,mnih2015human,schulman2015trust}, it has been shown that deep reinforcement learning policies are highly vulnerable to adversarial attacks. 
That is, a well-trained agent can produce wrong decisions under small perturbations, making it risky to deploy RL agents in real-life applications with noise and high stakes.
The most popular attack methods focus on fooling the RL agent by adversarially perturbing the states, actions, or transition dynamics of the victim~\citep{huang2017adversarial,pattanaik2017robust,zhang2020robust,sun2021strongest,tessler2019action}. 

However, in practice, many applications, such as air traffic control systems, are well-protected, meaning that direct perturbations to the observations or actions may not always be feasible. For instance, to perturb the readings of an air traffic control radar, an attacker might need to physically manipulate the sensor or infiltrate the communication system, tasks that can require significant effort. In this context, our paper explores attacks on a victim agent, $\nu$, executed by an attacker controlling another agent, $\alpha$, in the same environment. Specifically, in the air traffic control example, an attacker could manipulate a commercial drone to interfere with the radar system of the victim. The strategy employed by the attacker in this scenario is known as an “\textit{adversarial policy}”.

Previous works \citep{gleave2019adversarial,wu2021adversarial,guo2021adversarial} have adopted principled approaches to attack well-trained RL agents by developing an adversarial policy to directly minimize the expected return of the victim. These methods have effectively defeated state-of-the-art agents trained through self-play \citep{bansal2018emergent}, even when the adversarial policy has been trained for less than 3\% of the self-play training time steps. 
However, existing models do not adequately address situations where the attacker might face resistance and only achieve partial control, which also lead to conspicuous behaviors. Despite the common occurrence of such attacks, provably efficient defenses are not well investigated yet.



To address these issues, our generalized attack formulation introduces an ``attack budget'', effectively capturing the attacker's \textit{partial control}. This metric accurately reflects the attacker's capacity to degrade a victim's performance.
Within this framework, the attacker can self-regulate the attack budget, aligning with state distributions \citep{gleave2019adversarial} and marginalized transition dynamics \citep{russo2021balancing, franzmeyer2022illusionary}  to craft stealthier, less detectable attacks. Notably, our attack model extends single-agent action adversarial RL \citep{tessler2019action} to multi-agent setting.
On the defense side, merely retraining the victim agent against specific strong attacks may not necessarily improve overall robustness; in some cases, it could even worsen performance against other potential attacks. 
We propose an adversarial training algorithm featuring \textit{timescale separation}, which avoids overfitting to specific attacks and focuses on optimizing the agent's worst-case performance. 
Unlike existing methods of timescale separation in GANs and adversarial training in supervised learning, which may converge to local solutions \citep{heusel2017gans} or provide only empirical defenses \citep{madry2017towards,shafahi2019adversarial}, our algorithm converges to the most robust policy globally, offering defenses with provably efficient guarantees --- even in the face of the problem's non-convexity and non-concavity.

Our key contributions in the realm of both attack and defense are summarized as follows.
    \textbf{(1)} 
   We introduce a generalized attack formulation that captures the ``partial control'' of the attacker. This formulation allows for stealthier attacks and extends the concept of action adversarial RL to more generalized settings.
    \textbf{(2)} 
    We address the issue of non-convergence in adversarial training within our attack framework. By incorporating the principle of timescale separation, we achieve provable defenses and ensure theoretical guarantees for convergence to the globally most robust policy.
    \textbf{(3)} 
     Empirical results affirm the efficacy of our generalized attack formulation in minimizing state distribution shifts and generating stealthier behaviors, compared to baseline unconstrained methods. Additionally, in tasks like Kuhn Poker and Robosumo, our timescale-separated adversarial training demonstrates superior stability and robustness when compared to popular baselines, including single-timescale adversarial training, self-play, 
     and fictitious-play.


\vspace{-1em}
\section{Preliminaries}
\label{sec:prelim}
\vspace{-1em}

The extension of Markov decision processes (MDPs) with more than one agent is commonly modeled as Markov games~\citep{littman1994markov}. A Markov game with $N$ agents is defined by a tuple $\cG = <N, \mathcal{S}, \{\mathcal{A}_{i}\}_{i=1}^{N}, P, \{r_i\}_{i=1}^{N}, \rho, \gamma>$, where $\mathcal{S}$ denotes the state space and $\mathcal{A}_i$ is the action space for agent $i$. The function $P$ controls the state transitions by the current state and one action from each agent: $P: \mathcal{S}\times \mathcal{A}_1 \times \dots \times \mathcal{A}_N\rightarrow \Delta(\mathcal{S})$, where $\Delta(\mathcal{S})$ denotes the set of probability distributions over the state space $\mathcal{S}$. Given the current state $s_t$ and the joint action $(a_1, \dots, a_N)$, the transition probability to $s_{t+1}$ is given by $P(s_{t+1}|s_t, a_1, \dots , a_{N})$. The initial state is sampled from the initial state distribution $\rho \in \Delta(S)$. Each agent $i$ also has an associated reward function $r_i: \mathcal{S}\times \mathcal{A}_1\times \dots \times \mathcal{A}_N\rightarrow [0, 1]$, whose goal is to maximize the $\gamma$-discounted expected return $\mathbb{E}[\sum_{t=0}^{\infty} \gamma^{t}r_{i}(s_t, a_i^{t}, a_{-i}^{t})]$, where $-i$ is a compact representation of all complementary agents of $i$. 

In Markov games, each agent is equipped with a policy $\pi_i: \mathcal{S}\rightarrow \Delta(\mathcal{A}_i)$ in policy class $\Pi_{i}$ and the joint policy is defined as $\bm{\pi}(\mathbf{a}|s) = \Pi_{i=1}^{N}\pi_{i}(a_i|s)$. 
Specifically, the value function for the victim agent $\nu$
given joint policy $(\pi_\nu, \pi_\alpha)$ is defined by
$
    V_{s}(\pi_\nu, \pi_\alpha) = \mathbb{E}_{\pi_\nu, \pi_\alpha}\left[\sum_{t=0}^{\infty}\gamma^{t}r_{\nu}(s_t, \mathbf{a}_t)\mid s_0=s \right],
$
where agent $\nu$ attempts to maximize the value function and attacker aims to minimizes it. 
We abuse the notation to use $V_{\rho}(\pi_\nu, \pi_\alpha):= \mathbb{E}_{s\sim \rho}[V_{s}(\pi_\nu, \pi_\alpha)]$. 
We further define state visitation, which reflects how often the policy visits different states in the state space. 
\vspace{-0.5em}
\begin{definition}
(Stationary State Visitation)\label{def:state_dis} Let $d_\rho^{\bm{\pi}}\in \Delta(\mathcal{S})$ denote the normalized distribution of state visitation by following the joint policy $\bm{\pi}$ in the environment:
$d_\rho^{\bm{\pi}}(s) = (1-\gamma)\EE_{s_0\sim \rho}\sum_{t=0}^{\infty}\gamma^{t}P^{\bm{\pi}}(s_t=s|s_0)$.
\end{definition}
\vspace{-0.5em}



\vspace{-1em}
\section{A generalized attack formulation}
\vspace{-0.5em}

\textbf{Problem description.}  For simplicity, we consider a multi-agent system with two agents, $\nu$ and $\alpha$, following policies $\hat{\pi}_\nu$ and $\hat{\pi}_\alpha$ respectively. The interactions between these agents can be cooperative, competitive, or mixed. As motivated earlier, we consider the attack scenario as described in \cite{gleave2019adversarial, wu2021adversarial, guo2021adversarial}, where the threat comes from an attacker controlling agent $\alpha$. This attacker deviates from $\hat{\pi}_\alpha$ to an adversarial policy $\Tilde{\pi}_\alpha$, aiming to minimize the performance of the victim agent $\nu$. Correspondingly, the victim's goal is to develop a more robust policy $\pi_\nu$ in anticipation of such adversarial policies. The interaction between the attacker and the victim can thus be modeled as a zero-sum game, regardless of the initial relationship between the two agents.\footnote{Even if $\hat{\pi}_\nu$ and $\hat{\pi}_\alpha$ are trained competitively, successful attacks can still occur due to sub-optimality of training \cite{gleave2019adversarial,bansal2017emergent}.} 
This framework can also be extended to settings with more than two agents, where the attacker controls multiple agents and adopts a joint adversarial policy.

\textbf{Attack formulation.} Although such attacks can effectively exploit the victim, in many practical scenarios, unlike the attacks in \cite{gleave2019adversarial, wu2021adversarial, guo2021adversarial}, the attacker may face resistance and achieve only {\it partial} control of agent $\alpha$, e.g., in a hijack scenario. Therefore, we propose a more generalized attack framework. Here, the attacker aims to manipulate agent $\alpha$ using an adversarial policy $\Tilde{\pi}_\alpha$, but may not fully control the agent, which can still follow its original benign policy $\hat{\pi}_\alpha$ with probability $1-\epsilon_\pi$ at each time step, where $\epsilon_\pi\in [0, 1]$. Formally, under these conditions, the attacker solves the following attack objective:
\begin{align}
\label{eq:obj_new}&\min_{\Tilde{\pi}_{\alpha}}V_{\rho}(\hat{\pi}_{\nu}, \pi_{\alpha})\\
\label{eq:divergence_new}&\text{s.t.}\quad \pi_\alpha(\cdot\given s) = (1-\epsilon_\pi)\hat{\pi}_\alpha(\cdot\given s) + \epsilon_\pi \Tilde{\pi}_\alpha (\cdot\given s), \quad\forall s\in \cS.
\end{align}

Objective \ref{eq:obj_new} is a standard attack objective \citep{gleave2019adversarial,guo2021adversarial,wu2021adversarial}, focused on minimizing the value of the victim $\nu$. The additional constraint \ref{eq:divergence_new} captures the probability $\epsilon_\pi$ at which the attacker can control agent $\alpha$. When $\epsilon_\pi = 1$, the setting degenerates to full control, aligning with \cite{gleave2019adversarial,guo2021adversarial,wu2021adversarial}. Conversely, at $\epsilon_\pi = 0$, no attack occurs. This probability, denoted as the attack budget, effectively models the resistance encountered by the attacker. Its suitability as a budget will be further connected to action adversarial RL later \citep{tessler2019action}.

\textbf{(a) Effects of the attack budget.}
Our proposed attack budget effectively characterizes the victim's performance degradation, serving as a viable measure of agent vulnerability. To substantiate this, we note a key observation related to a standard discrepancy measure \citep{kakade2002approximately, schulman2015trust}. 
Given the constraint in Equation \ref{eq:divergence_new}, for any $\Tilde{\pi}_\alpha$, the inequality $D_{\operatorname{TV}}^{\max}(\pi_{\alpha}||\hat{\pi}_{\alpha})\le \epsilon_\pi$ holds.
Here, $D_{\operatorname{TV}}^{\max}(\pi_\alpha||\hat{\pi}_\alpha)$ is defined as $\max_{s}D_{\operatorname{TV}}(\pi_\alpha(\cdot|s)||\hat{\pi}_\alpha(\cdot|s))$, and $D_{\operatorname{TV}}(p||q) := \frac{1}{2}\sum_i |p_i-q_i|$. This observation allows us to establish an upper bound on the victim's performance under an attack budget of $\epsilon_\pi$.
\begin{proposition}
[\textit{Bounded policy discrepancy induces bounded value discrepancy}]\label{theorem:1}
For two policy pairs $(\hat{\pi}_{\nu}, \hat{\pi}_{\alpha})$ and $(\hat{\pi}_{\nu}, \pi_{\alpha})$ such that $D_{\operatorname{TV}}^{\max}(\pi_{\alpha}||\hat{\pi}_{\alpha})\le \epsilon_\pi$, the difference between the victim value can be bounded as:
$
    |V_{\rho}(\hat{\pi}_\nu, \hat{\pi}_\alpha) - V_{\rho}(\hat{\pi}_\nu, \pi_\alpha)|\le\frac{2\epsilon_\pi}{(1-\gamma)^2}.
$
\end{proposition}
\vspace{-1em}
This establishes a link between the attack budget and $|V_{\rho}(\hat{\pi}_\nu, \hat{\pi}_\alpha) - V_{\rho}(\hat{\pi}_\nu, \pi_\alpha)|$.
Specifically, the value function inherently satisfies a global Lipschitz condition. 
This implies that the attacker needs a sufficiently large attack budget $\epsilon_\pi$
to cause significant degradation in performance. 
This contrasts with supervised learning attacks \textcolor{black}{at test time}, where small perturbations can result in large performance shifts. Although this is a worst-case upper bound and may not be tight --- especially when 
$\gamma$ is close to 1 --- it still indicates that a longer effective game horizon grants the attacker greater capacity to degrade the victim's performance.

\textbf{(b) Attack model's stealthiness and detectability.} 
While the unconstrained attack in \cite{gleave2019adversarial} significantly impairs the victim's performance, its overt nature makes it easy to detect even through static images. This deviates from the stealthy ethos of adversarial attacks in supervised learning \citep{goodfellow2014explaining}. In contrast, our generalized attack framework allows for partial control of an agent and enables stealthier attacks by regulating the attack budget $\epsilon_\pi$.
Specifically, leveraging insights that static images alone can reveal attacks, we use generative modeling techniques for distribution matching to align the state distributions $d_\rho^{\hat{\pi}_{\nu}, \hat{\pi}_{\alpha}}$ and $ d_\rho^{\hat{\pi}_{\nu}, \pi_{\alpha}}$ induced by $(\hat{\pi}_\nu, \hat{\pi}_\alpha)$ and $(\hat{\pi}_\nu, {\pi}_\alpha)$ respectively.
Though exact state visitation is difficult to compute, regulating $\epsilon_\pi$
  allows us manage the discrepancy between these distributions. We adopt total variation distance as our measure of discrepancy, offering the following guarantees.
\begin{proposition}[\textit{Bounded policy discrepancy induces bounded state distribution discrepancy}]\label{theorem:2_new}
Fix any $\epsilon_\pi\in [0, 1]$. For two policy pairs $(\hat{\pi}_{\nu}, \hat{\pi}_{\alpha})$ and $(\hat{\pi}_{\nu}, \pi_{\alpha})$ such that $D_{\operatorname{TV}}^{\max}(\pi_{\alpha}||\hat{\pi}_{\alpha})\le \epsilon_\pi$, the discrepancy between the state distributions can be bounded as:
$
 ||d_\rho^{\hat{\pi}_{\nu}, \hat{\pi}_{\alpha}}-d_\rho^{\hat{\pi}_{\nu}, \pi_{\alpha}}||_{1}\le \frac{2\gamma\epsilon_\pi}{1-\gamma}.
$
\end{proposition}

Proposition \ref{theorem:2_new} demonstrates that as long as $\epsilon_\pi$ is sufficiently small, the state distribution is well preserved, thus yielding images that are visually more similar to the original ones. This suggests that in practice, $\epsilon_\pi$ can be treated as a hyper-parameter, balancing the attacker's performance with stealthiness. Finally, comparing actions or rewards is also a viable method to detect potential attacks. However, in many practical multi-agent systems, agents are decentralized, and actions or rewards are private to each agent \citep{zhang2018fully}, not always available to humans aiming to detect potential adversarial attacks.

\textit{Comparison with single-agent stealthy attacks.} \cite{russo2021balancing, franzmeyer2022illusionary} consider stealthy attacks in a \textit{different} setting, involving adversarial state or action perturbations, within single-agent RL. Their concept of unstealthiness or detectability is predicated on the inconsistencies in the \textit{transition dynamics} when states or actions are adversarially perturbed, necessitating an accurate world model. However, in our scenario, even if such a world model is accessible, the (global) transition dynamics $P$ remain unaffected by the adversarial policy. Concurrently, comparing the \textit{marginalized} transition dynamics induced by $\hat{\pi}_\alpha$ and $\pi_\alpha$ is plausible \textit{from the perspective of the victim $\nu$}. Based on Proposition \ref{prop:detection}, inconsistencies in the marginalized transition dynamics can also be upper-bounded by the variation in the policy space, assuring low \textit{detectability} as considered in \cite{russo2021balancing, franzmeyer2022illusionary}. Thus, even if humans or detectors can access the private actions of the agent $\nu$ and establish accurate corresponding marginalized transition dynamics, discrepancies might remain undetected as long as $\epsilon_\pi$ is maintained minimal.

\begin{proposition}[\textit{Bounded policy discrepancy induces bounded marginalized transition dynamics inconsistencies}]\label{prop:detection}
We define the {\it marginalized} transition dynamic of agent $\nu$ as $P^{\pi_\alpha}_\nu(s^\prime\given s, a_\nu) := \EE_{a_\alpha\sim \pi_\alpha(\cdot\given s)}[P(s^\prime\given s, a_\alpha, a_\nu)]$ for given $\pi_\alpha$. $P^{\hat{\pi}_\alpha}_\nu$ is defined similarly for the policy $\hat{\pi}_\alpha$. Then for any $s\in\cS$ and $a_\nu\in \cA_\nu$, we have
			$D_{f}\left(P^{\pi_\alpha}_\nu(\cdot \given s, a_\nu)\mid\mid P^{\hat{\pi}_\alpha}_\nu(\cdot \given s, a_\nu)\right)\le D_{f}\Big(\pi_\alpha(\cdot \given s)\mid\mid \hat{\pi}_\alpha(\cdot \given s)\Big)$,
	where $D_f$ is any $f$-divergence, which includes $D_{\operatorname{TV}}$, connecting back to the attack budget.
\end{proposition}
\textbf{(c) Connection to action adversarial RL.} Intriguingly, our attack formulation also extends the single-agent action adversarial RL \citep{tessler2019action} to a multi-agent setting. Specifically, in PR-MDP (cf. Definition 1 of \cite{tessler2019action}), the policy under attack aligns with our Equation \ref{eq:divergence_new}. In the context of single-agent action adversarial RL, the policy $\tilde{\pi}_\alpha$ is only a \textit{part} of the finally executed policy, while the policy $\hat{\pi}_\alpha$ represents the victim. Thus, our formulation broadens the attack setting of \cite{tessler2019action} to multi-agent RL, considering the \textit{other} agent $\nu$ as the victim instead of the agent $\alpha$ itself. Moreover, determining the most robust policy for the victim using the policy iteration scheme for PR-MDP from \cite{tessler2019action} becomes inefficient in our context due to the absence of specific structures inherent in PR-MDP (Section 4 of \cite{tessler2019action}). 

Henceforth, we will abbreviate $\tilde{\pi}_\alpha$ as $\pi_\alpha$ without ambiguity, and the actually deployed policy for agent $\alpha$ is represented as $(1-\epsilon_\pi)\hat{\pi}_{\alpha} + \epsilon_\pi \pi_\alpha$. Detailed proofs and discussions related to this section are available in \S\ref{sec:full_proof}.

\section{Improved adversarial training with timescale separation}
\label{sec:defense}
 \textbf{On the necessity and challenge of provably efficient defenses.} As discussed before, to provide effective defenses, there are unique challenges standing out compared with single-agent robust RL \citep{tessler2019action}. 
 Meanwhile, finding the celebrated solution concept Nash Equilibrium (NE) between the attacker and the victim suffices for finding the most robust victim policy during robust training but may not be necessary since NE guarantees that the attacker is also non-exploitable. 
 We provide more detailed discussions on the relationship between NE and robustness in \S\ref{sec:ne_robust}. There are a bunch of existing works solving NE for structured extensive-form games \citep{lockhart2019computing,brown2019deep,sokota2022unified} or for general games but without provably efficient guarantees \citep{fudenberg1995consistency,lanctot2017unified, balduzzi2019open,muller2019generalized}. In practice, general game-theoretical methods often require solving best response problems iteratively, thus being computationally expensive. In theory, simply plugging in {\it black-box} NE solvers may not solve our problem with provable efficiencies since finding even {\it local} NE for a {\it general} nonconvex-nonconcave problem is computationally hard \citep{daskalakis2021complexity}. Therefore, instead of adopting a black-box game-theoretical solver, we investigate adversarial training, a popular and more efficient paradigm for robust RL \citep{pinto2017robust,zhang2021robust,sun2021strongest}.

There are prior works that utilize well-trained attacks for re-training to fortify the robustness of the victim \citep{gleave2019adversarial, guo2021adversarial, wu2021adversarial}. However, it has been demonstrated that while re-training against a specific adversarial policy does augment robustness against it, the performance against other policies may be compromised \textcolor{black}{as validated by \cite{gleave2019adversarial}}. \textcolor{black}{Intuitively, if the victim is retrained against a specific attacker, its policy might be overfitted to that attacker.} Thus, it is vital to uphold the performance of the victim against all potential attackers. Rather than merely re-training against a specific attacker, adversarial training methods have been shown to be effective in bolstering robustness against a broad spectrum of adversarial attacks. In these methods, the victim and the attacker are trained alternatively or simultaneously \citep{zhang2020robust, pinto2017robust}. Here, we re-examine adversarial training in the RL domain and demonstrate that prevalent adversarial training methods encounter a \textit{non-converging} problem with either alternative or simultaneous training, for which we defer examples and detailed discussions to \S\ref{sec:motivation}. To address these issues formally, we contemplate the robustness of the victim and define the exploitability of $\pi_{\nu}$ under the \textit{worst-case} attack as follows.
\begin{definition}[{(One-side) exploitability}]
Given $\epsilon_\pi\in [0, 1]$ and $\hat{\pi}_\alpha$, for a victim policy $\pi_{\nu}$, we measure the robustness of $\pi_{\nu}$ by:
$
\operatorname{Expl}(\pi_{\nu}) = -\min_{\pi_\alpha}V_{\rho}(\pi_{\nu}, (1-\epsilon_\pi)\hat{\pi}_\alpha + \epsilon_\pi\pi_{\alpha}).
$
\end{definition}
\textbf{Intuitions of timescale separation.} The smaller $\operatorname{Expl}(\pi_{\nu})$ is, the more robust $\pi_\nu$ is.
Therefore, to ensure the worst-case performance against the strongest adversarial policy, the victim should optimize the policy according to $\min_{\pi_{\nu}}\operatorname{Expl}(\pi_\nu)$. Ideally, if we can derive an analytical form of the function $\operatorname{Expl}(\cdot)$ or compute its gradient, then we can simply run gradient descent to optimize it. Unfortunately, it is not obvious how to derive an analytical form and the function may not be even differentiable, let alone computing the gradient since the function relies on solving a minimization problem. However, it is possible to first solve the minimization problem, getting $\pi_\alpha^\star$, and compute the gradient w.r.t $\pi_\nu$, namely $\nabla_{\pi_\nu}-V_{\rho}(\pi_{\nu}, (1-\epsilon_\pi)\hat{\pi}_\alpha + \epsilon_\pi\pi_{\alpha}^\star)$, {\it as if $\pi_\alpha^\star$ is fixed, hoping it could serve as a good descent direction}. Formalizing this intuition, we propose to improve adversarial training via {\it timescale separation} with Min oracle (shown in Algorithm \ref{alg:max}), where timescale separation comes from the fact the attacker takes a min step against the victim in line $3$ while the victim takes only one gradient update in line $4$. Note \cite{lockhart2019computing} also considers {\it directly} minimizing the exploitability function but the algorithm and analysis are only applicable to extensive-form games. \textcolor{black}{Finally, we remark that Algorithm \ref{alg:max} is consistent with the leader-follower update style that is developed for Stackelberg equilibrium in multi-agent RL \citep{gerstgrasser2023oracles}.}
\begin{algorithm}[!htbp]
      \caption{Adversarial Training with Min-oracle}
      \label{alg:max}
    \begin{algorithmic}[1] 
      \STATE {\bfseries Input:} random policy $\pi_{\nu}^{0}$, learning rate sequence $\{\eta^t\}$
      \FOR{$t=0$ {\bfseries to} $T$}
      \STATE $\pi_{\alpha}^{t}\leftarrow \arg\min_{\pi_\alpha}V_{\rho}(\pi_{\nu}^{t}, (1-\epsilon_\pi)\hat{\pi}_\alpha + \epsilon _\pi\pi_{\alpha})$.
      \STATE $\pi_{\nu}^{t + 1}\leftarrow \mathcal{P}_{\Pi_{\nu}}(\pi_{\nu}^t + \eta^t\nabla_{\pi_\nu}V_{\rho}(\pi_{\nu}^t, (1-\epsilon_\pi)\hat{\pi}_\alpha + \epsilon _\pi\pi_{\alpha}^t)$.  // projection onto the simplex
      \ENDFOR
      \STATE {\bfseries Output:} sample $\pi_{\nu}^{t}$ with probability proportional to $\eta^{t}$.
    \end{algorithmic}
\end{algorithm}

\textbf{Efficient approximation.} The min oracle used in Algorithm \ref{alg:max} can be implemented with standard RL algorithms like PPO. When the game has special structures like extensive-form games \citep{lockhart2019computing} or one agent has a substantially smaller state/action space, such min oracle can be even implemented quite efficiently. However, in general, to make one gradient update for agent $\nu$, agent $\alpha$ needs to compute a complete best response, which is computationally expensive in practice. To fix this issue, we utilize the idea of using a much faster update scale for agent $\alpha$ so that when agent $\nu$ performs the gradient update, the policy $\pi_\alpha^t$ is always and already an {\it approximate solution} of $\arg\min_{\pi_\alpha}V_{\rho}(\pi_{\nu}^t, (1-\epsilon_\pi)\hat{\pi}_\alpha + \epsilon_\pi\pi_{\alpha})$. Formally, {in addition to} Algorithm \ref{alg:max}, we present an {alternative efficient} Algorithm \ref{alg:tts}, where the min oracle is replaced by a gradient update with a larger step size and both agents only need to perform a gradient update independently. \textit{Therefore, our final algorithm is simple and compatible with standard RL algorithms, like PPO to implement the gradient update step for both agents, avoiding solving best responses at each iteration under popular game-theoretical approaches \citep{fudenberg1995consistency,lanctot2017unified, balduzzi2019open,muller2019generalized}}. To validate our intuitions and verify that our algorithms do provide provably efficient defenses, we shall prove the convergence guarantee of both algorithms in the next section.

\begin{algorithm}[!htbp]
      \caption{Adversarial Training with Two Timescales}
      \label{alg:tts}
    \begin{algorithmic}[1] 
      \STATE {\bfseries Input:} random policy $\pi_{\nu}^{0}$, $\pi_{\alpha}^0$, learning rate sequence $\{\eta_\nu^t\}$, $\{\eta_\alpha^t\}$, such that $\eta_\nu^t\ll \eta_{\alpha}^t$.
      \FOR{$t=0$ {\bfseries to} $T$}
      \STATE $\pi_{\alpha}^{t + 1}\leftarrow \mathcal{P}_{\Pi_{\alpha}}(\pi_{\alpha}^t - \eta_{\alpha}^t\nabla_{\pi_\alpha}V_{\rho}(\pi_{\nu}^t, (1-\epsilon_\pi)\hat{\pi}_\alpha+\epsilon_\pi\pi_{\alpha}^t))$.
      \STATE $\pi_{\nu}^{t + 1}\leftarrow \mathcal{P}_{\Pi_{\nu}}(\pi_{\nu}^t +\eta_{\nu}^t\nabla_{\pi_\nu}V_{\rho}(\pi_{\nu}^t, (1-\epsilon_\pi)\hat{\pi}_\alpha+\epsilon_\pi\pi_{\alpha}^t))$.
      \ENDFOR
      \STATE {\bfseries Output:} sample $\pi_{\nu}^{t}$ with probability proportional to $\eta_{\nu}^t$.
    \end{algorithmic}
\end{algorithm}

\section{Theoretical analysis}
\label{sec:theory}

To understand and verify our approaches, we start by considering direct policy parameterization for both agents $\nu$ and $\alpha$, which is already challenging due to the non-convexity, non-concavity, and serves as the first step to analyze more complex function approximations.
\begin{definition}[Direct parameterization]\label{def:param}
    The policies $\pi_\nu$ and $\pi_\alpha$ have the  parameterization
        $
        \pi_{\nu}(a\given s) = \nu_{s, a}$,
        $
        \pi_{\alpha}(a\given s) = \alpha_{s, a},
    $
    where $\nu\in \Delta(\cA_\nu)^{|\cS|}$ and $\alpha\in \Delta(\cA_{\alpha})^{|\cS|}$. 
\end{definition}
For convenience of discussions, given $\epsilon_\pi$ and $\hat{\pi}_{\alpha}$, we will write $J_{\epsilon_\pi}(\pi_\nu, \pi_\alpha):=V_{\rho}(\pi_\nu, (1-\epsilon_\pi)\hat{\pi}_\alpha +\epsilon_\pi\pi_\alpha)$.
Before proving the convergence of our methods, we define the mismatch coefficient to measure the intrinsic hardness of the environment. \textcolor{black}{This is achieved by comparing the stationary state occupancy frequencies under certain policies against the initial state distribution. In simpler terms, a smaller value of this quantity indicates that the environment is more easily explorable.}
\begin{definition}\label{assump:c_g}
Given the Markov game $\cG$, benign policy $\hat{\pi}_\alpha$, and attack budget $\epsilon_\pi$, we define the minimax mismatch coefficient as
\begin{small}
	\begin{align*}
    C_{\mathcal{G}}^{\epsilon_{\pi}}:=\max \left\{\max_{\pi_\nu\in \Pi_{\nu}} \min_{\pi_\alpha \in \Pi_\alpha^{\star}\left(\pi_\nu\right)}\left\|\frac{d_\rho^{\pi_\nu, (1-\epsilon_\pi)\hat{\pi}_\alpha + \epsilon_\pi\pi_\alpha}}{\rho}\right\|_{\infty},
    \max_{\pi_\alpha\in\Pi_\alpha} \min_{\pi_\nu \in \Pi_\nu^{\star}\left(\pi_\alpha\right)}\left\|\frac{d_\rho^{\pi_\nu, (1-\epsilon_\pi)\hat{\pi}_\alpha + \epsilon_\pi\pi_\alpha}}{\rho}\right\|_{\infty}\right\},
\end{align*}
\end{small}
where $\Pi_\alpha^{\star}\left(\pi_\nu\right):=\arg\min_{\pi_\alpha\in \Pi_{\alpha}} J_{\epsilon_\pi}(\pi_\nu, \pi_\alpha)$, and $\Pi_\nu^{\star}\left(\pi_\alpha\right):=\arg\max_{\pi_\nu\in \Pi_{\nu}} J_{\epsilon_\pi}(\pi_\nu, \pi_\alpha)$. 
\end{definition}
With those two definitions, we can analyze how the robustness of the victim improves during adversarial training as follows:
\begin{theorem}\label{thm:max}
Fix any $\delta>0, \epsilon_\pi\in [0, 1]$. For Algorithm \ref{alg:max}, suppose the learning rate $\eta_\nu^t \asymp \delta$, after $T$ iterations, it is guaranteed that
$\frac{1}{T}\sum_{t=1}^T \operatorname{Expl}(\pi_\nu^t)\le \min_{\pi_\nu}\operatorname{Expl}(\pi_\nu)+\delta,$ 
    where $T = \frac{1}{\delta^2}\operatorname{poly}( C_\cG^{\epsilon_\pi}, |\cS|, |\cA_\alpha|, |\cA_\nu|, \frac{1}{1-\gamma})$; 
    while for Algorithm \ref{alg:tts}, suppose the learning rate $\eta_\nu^t \asymp \delta^{8}$, $\eta_\alpha^t \asymp \delta^{4}$, after $T$ iterations, it is guaranteed that
$\frac{1}{T}\sum_{t=1}^T \operatorname{Expl}(\pi_\nu^t)\le \min_{\pi_\nu}\operatorname{Expl}(\pi_\nu)+\delta,$ 
    where $T = \operatorname{poly}(\frac{1}{\delta}, C_\cG^{\epsilon_\pi}, |\cS|, |\cA_\alpha|, |\cA_\nu|, \frac{1}{1-\gamma})$.
\end{theorem}


\begin{remark}
	To get a non-vacuous finite time convergence, the mismatch coefficient needs to be bounded, which is standard and necessary in the analysis of policy gradient methods \citep{daskalakis2020independent, agarwal2021theory}, where our definition of $C_\cG^{\epsilon_\pi}$ is based on the definition in \citep{daskalakis2020independent}. It is worth noticing that such an assumption is weaker than other similar notions such as concentrability \citep{munos2003error, chen2019information}, without requiring every visitable state to be visited at the first time step. 
\end{remark}
\textbf{Implications.} This theorem demonstrates that, \textit{on average}, the victim policy $\pi_{\nu}^t$ is assured to converge to the \textit{most robust} one; that is, the solution of $\arg\max_{\pi_\nu}\min_{\pi_\alpha} V_\rho(\pi_\nu, (1-\epsilon_\pi)\hat{\pi}\alpha + \epsilon\pi\pi_\alpha)$.
Theorem \ref{thm:max} reveals that Algorithm \ref{alg:max} achieves better iteration complexity owing to a larger learning rate. Meanwhile, the convergence for Algorithm \ref{alg:tts} also substantiates the necessity of timescale separation due to $\eta_\nu^t\ll \eta_\alpha^t$. To the best of our knowledge, the analysis for Algorithm \ref{alg:max} is new, even for the $\epsilon_\pi = 1$ case, and the analysis for Algorithm \ref{alg:tts} leverages \cite{daskalakis2020independent,jin2020local,zhang2021gradient} but addresses the more prevalent discounted reward setting \textit{without} assuming the game to cease at every state with a positive probability. It also permits the attacker to exert only partial control over the agent $\alpha$ within any stipulated attack budget $\epsilon_\pi$. Importantly, our theorem can be readily extended to encompass results of \textit{last-iterate} convergence using the regularization techniques highlighted in \cite{zeng2022regularized} and \textit{stochastic} gradients by deploying popular gradient estimators from finite samples, with the central concept remaining timescale separation. Detailed proofs are available in \S\ref{sec:full_proof}.

\begin{figure*}[!t]
\centering
\includegraphics[width=0.99\textwidth]{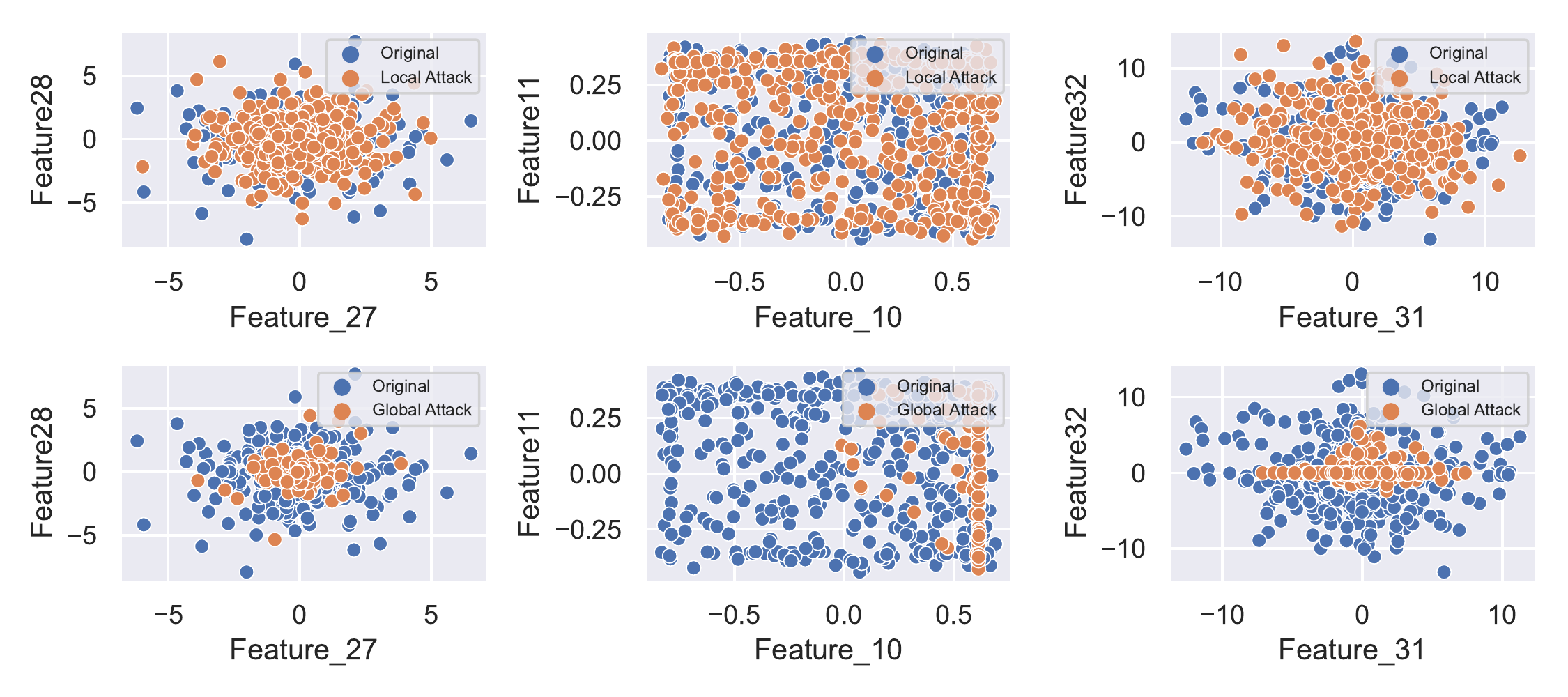}
\vspace{-5mm}
\caption{
{Visualization and comparison of our proposed constrained attack with $\epsilon_\pi=0.2$ (first row) vs. an unconstrained attack (second row, $\epsilon_\pi = 1$), under the condition that both achieve the same attacking success rate.}
The most important state features are shown.
It is clear that our constrained adversarial policy induces much smaller state distribution shifts.}
\label{fig:feat_imp}
 \vspace{-1em}
\end{figure*}
\vspace{-0.5em}
\section{Related work}\label{sec:related}
 \vspace{-0.5em}
\textbf{Stealthy adversarial attacks in RL.} 
To render attacks on RL policies more feasible and practical, \cite{sun2020stealthy} demonstrates that targeting critical points can facilitate efficient and stealthy attacks. \cite{russo2021balancing} optimizes both attack detectability and victim performance, analyzing the trade-off between them. \cite{franzmeyer2022illusionary} introduces an illusionary attack that maintains consistency with the environment’s transition dynamics, necessitating a world model. The aforementioned stealthy attacks focus on perturbing the victim's state observations/actions. In contrast, this paper presents a generalized attack framework for multi-agent systems, allowing for stealth by managing the attack budget, wherein the attacker indirectly influences the victim by altering another agent’s policy.

\textbf{Timescale separation for adversarial training.} Adversarial training is a widely-adopted method for cultivating models robust against adversarial attacks. The efficacy of timescale separation in this context has been empirically affirmed; having increased loops for the inner attack subroutine translates to enhanced robustness \cite{madry2017towards,shafahi2019adversarial}. Likewise, in training GANs \citep{heusel2017gans}, utilizing a larger learning rate for the discriminator surpasses conventional GAN training and ensures convergence to a local NE. Moreover, \cite{fiez2021local} explores more general non-convex non-concave zero-sum games and elucidates the local convergence to strict local minimax equilibrium with finite timescale separation. Contrarily, our adversarial training algorithm is assured to converge to the (globally) most robust policy.
 


\vspace{-0.5em}
\section{Experiments}

\begin{figure}[!t]
  \centering
\begin{subfigure}{\textwidth}
\centering
\includegraphics[width=.9\textwidth]{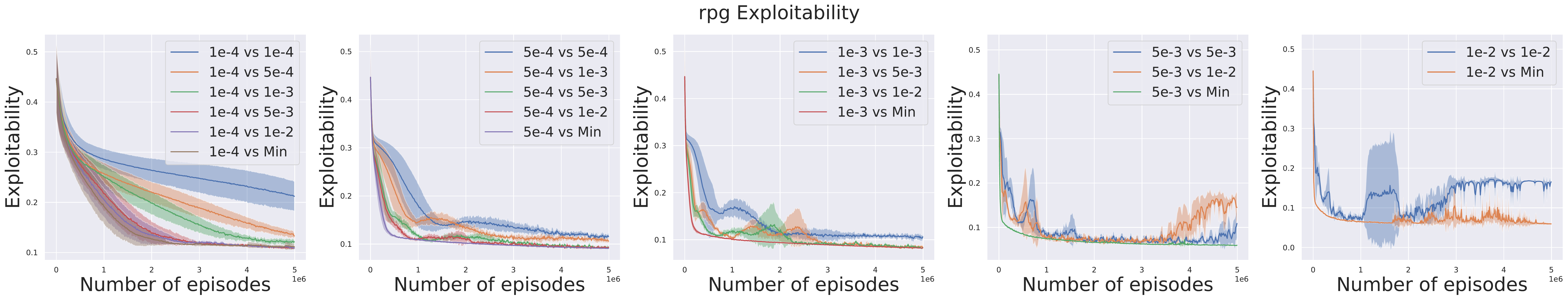}
\caption{Exploitability of the victim when using RPG for policy gradient.}
\end{subfigure}
\begin{subfigure}{\textwidth}
\centering
\includegraphics[width=.9\textwidth]{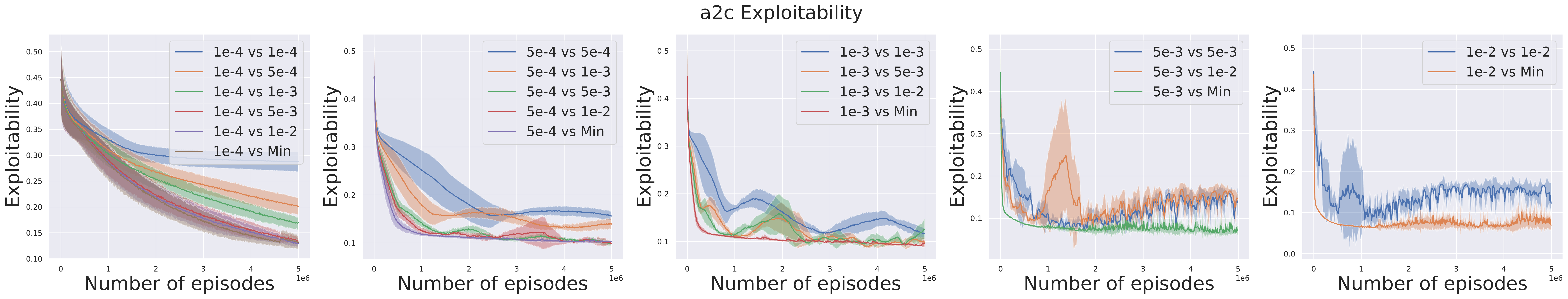}
\caption{Exploitability of the victim when using A2C for policy gradient.}
\end{subfigure}
\vspace{-5mm}
\caption{Exploitability of victim policy in Kuhn Poker trained by two timescale and single timescale (min indicates the policy trained with a min oracle).}
\label{exp:kuhn_poker}
\vspace{-2em}
\end{figure}

\vspace{-0.5em}
Our experiments utilize two standard environments: Kuhn Poker \citep{kuhn1950simplified,lanctot2019openspiel} and RoboSumo \citep{al2017continuous}. Detailed introductions to these environments, implementation specifics, and hyper-parameters are outlined in \S\ref{app:exp}. Unless specified otherwise, the results are averaged over $5$ seeds. In this section, we seek to address the following pivotal questions:
\begin{list}{$\rhd$}{\topsep=-1.ex \leftmargin=0.19in \rightmargin=0.in \itemsep=-0.05in}
\item \textbf{Q1.} Can our generalized attack formulation produce fewer state distribution variations and more stealthy behaviors compared to an unconstrained attack?
\item \textbf{Q2.} Will adversarial training with timescale separation (involving a min oracle and two timescales) exhibit more stable learning dynamics, and can the two timescales algorithm effectively approximate the adversarial training with a min oracle?
\item \textbf{Q3.} In complex environments, where a min oracle may not be available, can adversarial training with two timescales enhance robustness compared to prevalent and acclaimed baselines?
\end{list}



\textbf{Controllable adversarial attack (Q1).}
We conduct experiments on the Robosumo environment. 
To verify that our attack formulation indeed achieves smaller state distribution variations, we choose an unconstrained adversarial policy and a constrained policy with the same winning rate for fair comparisons. 
Specifically, we investigate the distribution shift in the victim's observation part of the state features. 
\begin{wrapfigure}{r}{0.4\columnwidth}
    \centering
    \vspace{-0.5em}
    \includegraphics[width=0.4\columnwidth]{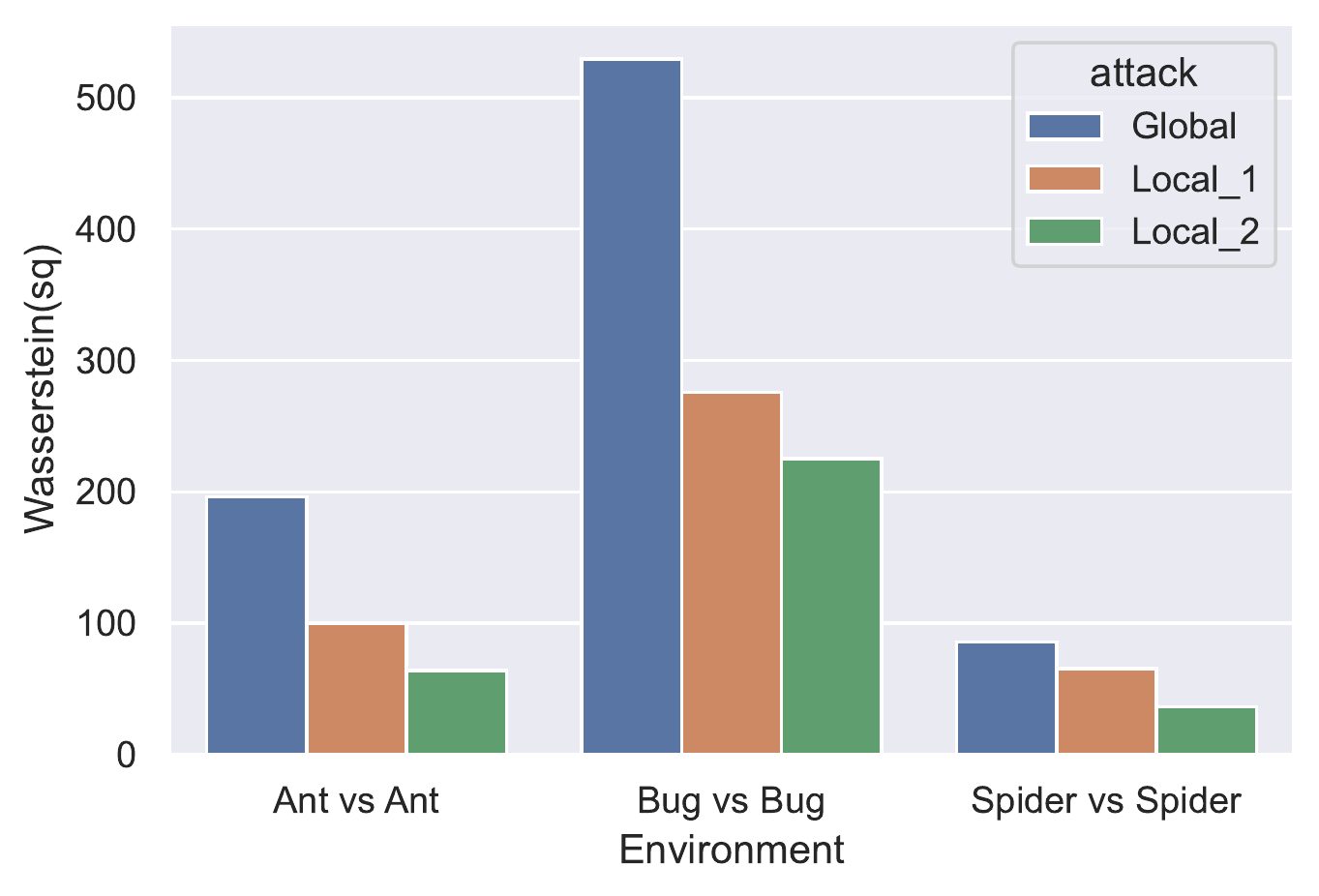}
    \vspace{-2em}
    \caption{State-distribution shift w.r.t Wasserstein-2 distance (squared) incurred due to the Global ($\epsilon_\pi = 1$), Local\_1 ($\epsilon_\pi = 0.7$),  Local\_2($\epsilon_\pi = 0.3$) attacks in $3$ Robosumo environments. 
    } 
    \label{fig:wasserstein_plot}
\vspace{-0.5em}
\end{wrapfigure}
To select essential state features, we employ the variance-based feature importance method to filter out state features with small variances as they are deemed unimportant. The sorted feature importance for the RoboSumo games is depicted in Figure \ref{app:feat_imp}. Figure \ref{fig:feat_imp} demonstrates that the adversarial policy, derived from our generalized attack framework, induces a much smaller state distribution shift compared to the unconstrained adversarial policy when assessed under the same winning rate. Additionally, we quantify the state-distribution shift brought about by the constrained and unconstrained attacks by calculating the Wasserstein-2 distance between their state distributions, as illustrated in Figure \ref{fig:wasserstein_plot}.
\textbf{Our constrained attack results in a significantly lower state distribution shift compared to the unconstrained one.} To confirm that our generalized attack methods, with a regulated attack budget, do indeed produce more stealthy behaviors, we visualize agents with $\epsilon_\pi \in {0.3, 0.7, 1}$; \textbf{here, a smaller $\epsilon_\pi$ induces behaviors that are visually more similar to the system without an attack} (see gifs at \href{https://sites.google.com/view/stealthy-attack}{https://sites.google.com/view/stealthy-attack}). An ablation study on the trade-off between stealthiness and the attacker's performance is also presented in \S\ref{addi_res}, also validating the effectiveness of $\epsilon_\pi$ in regulating the attacker's strength, in line with Proposition \ref{theorem:1}. We show in \S\ref{addi_res}, \textbf{even when there is a mismatch for attack budgets between training time and test time, the victim policy still shows greatly improved robustness.}

\begin{figure}[!t]
  \centering
\begin{subfigure}{0.95\textwidth}
\begin{center}
\includegraphics[width=.99\textwidth]{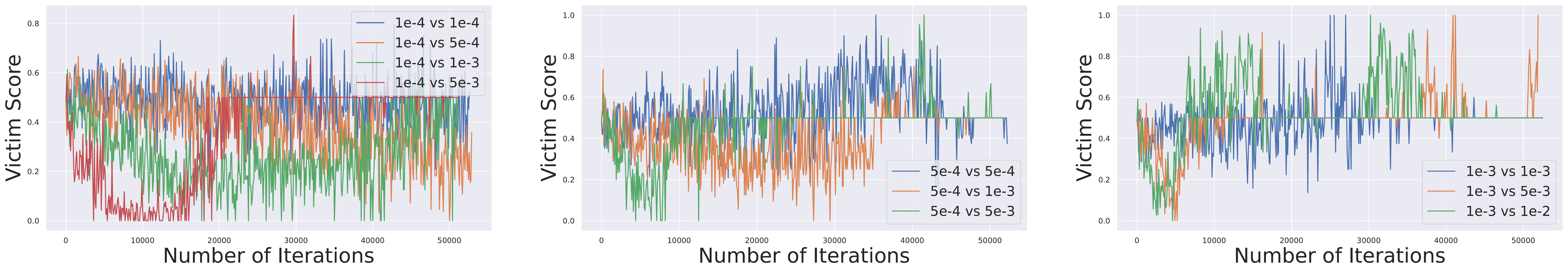}
\end{center}
\vspace{-2mm}
\caption{Score of the victim policy, which is computed by winning rate + tie rate/2. Note here different from other plots, to show the potential oscillation behaviors during training, we show just one seed instead of multiple ones.}
\label{fig:score}
\end{subfigure}
\hfill 
\begin{subfigure}{0.95\textwidth}
\begin{center}
\includegraphics[width=.99\textwidth]{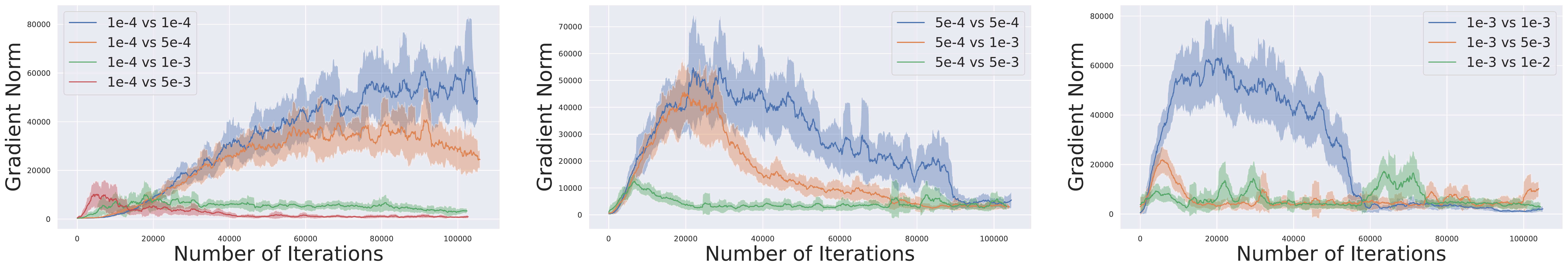}
\end{center}
\vspace{-2mm}
\caption{Norm of the gradient $\nabla_{\pi_\nu}V_{\rho}(\pi_\nu^t, \pi_\alpha^t)$.}
\label{fig:norm}
\end{subfigure}
\hfill
\begin{subfigure}{0.95\textwidth}
\begin{center}
\includegraphics[width=.99\textwidth]{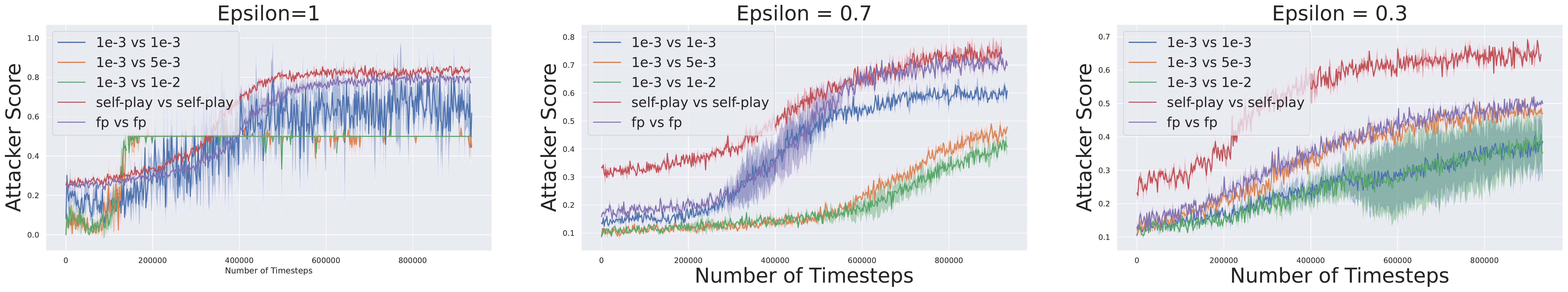}
\end{center}
\vspace{-2mm}
\caption{Score of the attacker against robustified policy (final output of adversarial training). Lower is better.}
\label{fig:attacK_score}
\end{subfigure}
\caption{(a). The score of the victim policy trained with two timescales converges rapidly, while the policy trained only with a single timescale suffers from much more oscillations. (b). The gradient norm trained by two timescales is also much smaller. (c). Under different $\epsilon_\pi = 1, 0.7, 0.3$, when attacking the robustified victim policy, i.e. computing $\min_{\pi_{\alpha}}V_{\rho}(\pi_{\nu}^{\star}, (1-\epsilon_\pi)\hat{\pi}+\epsilon_\pi\pi_\alpha)$ with standard RL algorithm, victim trained by two timescales achieves the lowest exploitability/best robustness.}
\label{exp:robosumo}
\vspace{-7mm}

\end{figure}

\textbf{Adversarial Training with Timescale Separation (Q2\&3).}
To address Q2\&3, we examine the learning dynamics and robustness of policies trained with and without timescale separation, comparing these to other baselines in both Kuhn Poker and Robosumo environments.

\emph{Kuhn Poker.} We implement Algorithm \ref{alg:max} and \ref{alg:tts} under OpenSpiel \citep{lanctot2019openspiel}, with the min oracle achieved through game tree search. For gradient update, we utilize Regret Policy Gradient (RPG) \citep{srinivasan2018actor} and Advantage Actor-Critic (A2C) \citep{mnih2016asynchronous} in lieu of the vanilla policy gradient. Figure \ref{exp:kuhn_poker} illustrates the exploitability of the victim policy $\pi_\nu$, \textbf{where adversarial training with the min oracle exhibits the fastest convergence, lowest exploitability, and least variance.} Meanwhile, policies trained with a sufficient timescale separation parameter ${\eta_{\alpha}^t}/{\eta_{\nu}^t}$ closely approximate the algorithm with a min oracle, outperforming single timescale algorithms. 

\emph{Robosumo Competition.} To demonstrate the scalability of timescale separation, we evaluate our methods on Robosumo—a high-dimensional, continuous control task, representing a significant challenge in terms of both training and evaluation. Although a min oracle with game tree search is unattainable in such a continuous control task, earlier experiments suggest that an ample timescale separation ratio can effectively approximate it. We monitor the score, $V_{\rho}(\pi_{\nu}^t, (1-\epsilon_\pi)\hat{\pi}_\alpha + \epsilon_\pi\pi_{\alpha}^t)$, and the norm of gradient $\nabla_{\pi_\nu}V_{\rho}(\pi_\nu^t, (1-\epsilon_\pi)\hat{\pi}_\alpha + \epsilon_\pi\pi_\alpha^t)$ during adversarial training. As per Figure \ref{fig:score} and \ref{fig:norm}, we ascertain that \textbf{single timescale training results in unstable behaviors, large variance, and gradient norm, while the two timescale training achieves quick convergence and significantly smaller gradient norm.} Crucially, to affirm the robustness of our methods, we compare them with single timescale adversarial training, self-play \citep{bansal2017emergent}, and fictitious-play \citep{heinrich2016deep}. Self-play and fictitious-play-based methods are renowned for training adversarially robust RL agents \citep{pinto2017robust,zhang2021robust,tessler2019action}. We employ PPO to calculate the best response of the finalized robustified victim policy, illustrating the performance of the attacker during the victim attack process in Figure \ref{fig:attacK_score}. Here, a lower winning rate for the attacker signifies enhanced robustness. \textbf{This demonstrates that our adversarial training with two timescales leads to victim policies with enhanced robustness compared to popular baseline methods, single timescale adversarial training, self-play, and fictitious-play.}

 \vspace{-3mm}
 \section{Discussion and limitations} 
\label{sec:discuss}
 \vspace{-2mm}
In this paper, we reassess the threats posed to RL agents by adversarial policies by introducing a generalized attack formulation and develop the first provably efficient defense algorithm, “adversarial training with timescale separation”, with convergence to the most robust policy under mild conditions. Meanwhile, we leave how to scale our formulation to accommodate multiple independent attackers with self-interested adversarial policies as future works. 

\section{Acknowledgement}
Liu, Chakraborty, Sun, and Huang are supported by National Science Foundation NSF-IIS-FAI program, DOD-ONR-Office of Naval Research, DOD Air Force Office of Scientific Research, DOD-DARPA-Defense Advanced Research Projects Agency Guaranteeing AI Robustness against Deception (GARD), Adobe, Capital One and JP Morgan faculty fellowships. The authors also thank Kaiqing Zhang and Soheil Feizi for the valuable discussions at the initial stages.

\bibliography{0_references}

\begin{thebibliography}{78}
\providecommand{\natexlab}[1]{#1}
\providecommand{\url}[1]{\texttt{#1}}
\expandafter\ifx\csname urlstyle\endcsname\relax
  \providecommand{\doi}[1]{doi: #1}\else
  \providecommand{\doi}{doi: \begingroup \urlstyle{rm}\Url}\fi

\bibitem[Agarwal et~al.(2021)Agarwal, Kakade, Lee, and
  Mahajan]{agarwal2021theory}
Alekh Agarwal, Sham~M Kakade, Jason~D Lee, and Gaurav Mahajan.
\newblock On the theory of policy gradient methods: Optimality, approximation,
  and distribution shift.
\newblock \emph{J. Mach. Learn. Res.}, 22\penalty0 (98):\penalty0 1--76, 2021.

\bibitem[Al-Shedivat et~al.(2017)Al-Shedivat, Bansal, Burda, Sutskever,
  Mordatch, and Abbeel]{al2017continuous}
Maruan Al-Shedivat, Trapit Bansal, Yuri Burda, Ilya Sutskever, Igor Mordatch,
  and Pieter Abbeel.
\newblock Continuous adaptation via meta-learning in nonstationary and
  competitive environments.
\newblock \emph{arXiv preprint arXiv:1710.03641}, 2017.

\bibitem[Balduzzi et~al.(2019)Balduzzi, Garnelo, Bachrach, Czarnecki, Perolat,
  Jaderberg, and Graepel]{balduzzi2019open}
David Balduzzi, Marta Garnelo, Yoram Bachrach, Wojciech Czarnecki, Julien
  Perolat, Max Jaderberg, and Thore Graepel.
\newblock Open-ended learning in symmetric zero-sum games.
\newblock In \emph{International Conference on Machine Learning}, pp.\
  434--443. PMLR, 2019.

\bibitem[Bansal et~al.(2017)Bansal, Pachocki, Sidor, Sutskever, and
  Mordatch]{bansal2017emergent}
Trapit Bansal, Jakub Pachocki, Szymon Sidor, Ilya Sutskever, and Igor Mordatch.
\newblock Emergent complexity via multi-agent competition.
\newblock \emph{arXiv preprint arXiv:1710.03748}, 2017.

\bibitem[Bansal et~al.(2018)Bansal, Pachocki, Sidor, Sutskever, and
  Mordatch]{bansal2018emergent}
Trapit Bansal, Jakub Pachocki, Szymon Sidor, Ilya Sutskever, and Igor Mordatch.
\newblock Emergent complexity via multi-agent competition, 2018.

\bibitem[Blumenkamp \& Prorok(2020)Blumenkamp and
  Prorok]{blumenkamp2020emergence}
Jan Blumenkamp and Amanda Prorok.
\newblock The emergence of adversarial communication in multi-agent
  reinforcement learning, 2020.

\bibitem[Brown et~al.(2019)Brown, Lerer, Gross, and Sandholm]{brown2019deep}
Noam Brown, Adam Lerer, Sam Gross, and Tuomas Sandholm.
\newblock Deep counterfactual regret minimization.
\newblock In \emph{International conference on machine learning}, pp.\
  793--802. PMLR, 2019.

\bibitem[Chen \& Jiang(2019)Chen and Jiang]{chen2019information}
Jinglin Chen and Nan Jiang.
\newblock Information-theoretic considerations in batch reinforcement learning.
\newblock In \emph{International Conference on Machine Learning}, pp.\
  1042--1051. PMLR, 2019.

\bibitem[Cohen et~al.(2019)Cohen, Rosenfeld, and Kolter]{cohen2019certified}
Jeremy Cohen, Elan Rosenfeld, and Zico Kolter.
\newblock Certified adversarial robustness via randomized smoothing.
\newblock In \emph{International Conference on Machine Learning}, pp.\
  1310--1320. PMLR, 2019.

\bibitem[Daskalakis et~al.(2020)Daskalakis, Foster, and
  Golowich]{daskalakis2020independent}
Constantinos Daskalakis, Dylan~J Foster, and Noah Golowich.
\newblock Independent policy gradient methods for competitive reinforcement
  learning.
\newblock \emph{Advances in neural information processing systems},
  33:\penalty0 5527--5540, 2020.

\bibitem[Daskalakis et~al.(2021)Daskalakis, Skoulakis, and
  Zampetakis]{daskalakis2021complexity}
Constantinos Daskalakis, Stratis Skoulakis, and Manolis Zampetakis.
\newblock The complexity of constrained min-max optimization.
\newblock In \emph{Proceedings of the 53rd Annual ACM SIGACT Symposium on
  Theory of Computing}, pp.\  1466--1478, 2021.

\bibitem[Fiez \& Ratliff(2021)Fiez and Ratliff]{fiez2021local}
Tanner Fiez and Lillian~J Ratliff.
\newblock Local convergence analysis of gradient descent ascent with finite
  timescale separation.
\newblock In \emph{Proceedings of the International Conference on Learning
  Representation}, 2021.

\bibitem[Fischer et~al.(2019)Fischer, Mirman, Stalder, and
  Vechev]{fischer2019online}
Marc Fischer, Matthew Mirman, Steven Stalder, and Martin Vechev.
\newblock Online robustness training for deep reinforcement learning.
\newblock \emph{arXiv preprint arXiv:1911.00887}, 2019.

\bibitem[Franzmeyer et~al.(2022)Franzmeyer, Henriques, Foerster, Torr, Bibi,
  and de~Witt]{franzmeyer2022illusionary}
Tim Franzmeyer, Jo{\~a}o~F Henriques, Jakob~N Foerster, Philip~HS Torr, Adel
  Bibi, and Christian~Schroeder de~Witt.
\newblock Illusionary attacks on sequential decision makers and
  countermeasures.
\newblock \emph{arXiv preprint arXiv:2207.10170}, 2022.

\bibitem[Fudenberg \& Levine(1995)Fudenberg and
  Levine]{fudenberg1995consistency}
Drew Fudenberg and David~K Levine.
\newblock Consistency and cautious fictitious play.
\newblock \emph{Journal of Economic Dynamics and Control}, 19\penalty0
  (5-7):\penalty0 1065--1089, 1995.

\bibitem[Gerstgrasser \& Parkes(2023)Gerstgrasser and
  Parkes]{gerstgrasser2023oracles}
Matthias Gerstgrasser and David~C Parkes.
\newblock Oracles \& followers: Stackelberg equilibria in deep multi-agent
  reinforcement learning.
\newblock In \emph{International Conference on Machine Learning}, pp.\
  11213--11236. PMLR, 2023.

\bibitem[Gleave et~al.(2019)Gleave, Dennis, Wild, Kant, Levine, and
  Russell]{gleave2019adversarial}
Adam Gleave, Michael Dennis, Cody Wild, Neel Kant, Sergey Levine, and Stuart
  Russell.
\newblock Adversarial policies: Attacking deep reinforcement learning.
\newblock \emph{arXiv preprint arXiv:1905.10615}, 2019.

\bibitem[Goodfellow et~al.(2014)Goodfellow, Shlens, and
  Szegedy]{goodfellow2014explaining}
Ian~J Goodfellow, Jonathon Shlens, and Christian Szegedy.
\newblock Explaining and harnessing adversarial examples.
\newblock \emph{arXiv preprint arXiv:1412.6572}, 2014.

\bibitem[Goodfellow et~al.(2015)Goodfellow, Shlens, and
  Szegedy]{goodfellow2015explaining}
Ian~J. Goodfellow, Jonathon Shlens, and Christian Szegedy.
\newblock Explaining and harnessing adversarial examples.
\newblock In Yoshua Bengio and Yann LeCun (eds.), \emph{3rd International
  Conference on Learning Representations, {ICLR} 2015, San Diego, CA, USA, May
  7-9, 2015, Conference Track Proceedings}, 2015.
\newblock URL \url{http://arxiv.org/abs/1412.6572}.

\bibitem[Gowal et~al.(2018)Gowal, Dvijotham, Stanforth, Bunel, Qin, Uesato,
  Arandjelovic, Mann, and Kohli]{gowal2018effectiveness}
Sven Gowal, Krishnamurthy Dvijotham, Robert Stanforth, Rudy Bunel, Chongli Qin,
  Jonathan Uesato, Relja Arandjelovic, Timothy Mann, and Pushmeet Kohli.
\newblock On the effectiveness of interval bound propagation for training
  verifiably robust models.
\newblock \emph{arXiv preprint arXiv:1810.12715}, 2018.

\bibitem[Gowal et~al.(2019)Gowal, Dvijotham, Stanforth, Bunel, Qin, Uesato,
  Arandjelovic, Mann, and Kohli]{gowal2019scalable}
Sven Gowal, Krishnamurthy~Dj Dvijotham, Robert Stanforth, Rudy Bunel, Chongli
  Qin, Jonathan Uesato, Relja Arandjelovic, Timothy Mann, and Pushmeet Kohli.
\newblock Scalable verified training for provably robust image classification.
\newblock In \emph{Proceedings of the IEEE/CVF International Conference on
  Computer Vision}, pp.\  4842--4851, 2019.

\bibitem[Guo et~al.(2021)Guo, Wu, Huang, and Xing]{guo2021adversarial}
Wenbo Guo, Xian Wu, Sui Huang, and Xinyu Xing.
\newblock Adversarial policy learning in two-player competitive games.
\newblock In \emph{International Conference on Machine Learning}, pp.\
  3910--3919. PMLR, 2021.

\bibitem[Hayes(2020)]{hayes2020extensions}
Jamie Hayes.
\newblock Extensions and limitations of randomized smoothing for robustness
  guarantees.
\newblock In \emph{Proceedings of the IEEE/CVF Conference on Computer Vision
  and Pattern Recognition Workshops}, pp.\  786--787, 2020.

\bibitem[Heinrich \& Silver(2016)Heinrich and Silver]{heinrich2016deep}
Johannes Heinrich and David Silver.
\newblock Deep reinforcement learning from self-play in imperfect-information
  games.
\newblock \emph{arXiv preprint arXiv:1603.01121}, 2016.

\bibitem[Heusel et~al.(2017)Heusel, Ramsauer, Unterthiner, Nessler, and
  Hochreiter]{heusel2017gans}
Martin Heusel, Hubert Ramsauer, Thomas Unterthiner, Bernhard Nessler, and Sepp
  Hochreiter.
\newblock Gans trained by a two time-scale update rule converge to a local nash
  equilibrium.
\newblock \emph{Advances in neural information processing systems}, 30, 2017.

\bibitem[Huang et~al.(2017)Huang, Papernot, Goodfellow, Duan, and
  Abbeel]{huang2017adversarial}
Sandy Huang, Nicolas Papernot, Ian Goodfellow, Yan Duan, and Pieter Abbeel.
\newblock Adversarial attacks on neural network policies, 2017.

\bibitem[Jin et~al.(2020)Jin, Netrapalli, and Jordan]{jin2020local}
Chi Jin, Praneeth Netrapalli, and Michael Jordan.
\newblock What is local optimality in nonconvex-nonconcave minimax
  optimization?
\newblock In \emph{International conference on machine learning}, pp.\
  4880--4889. PMLR, 2020.

\bibitem[Kakade \& Langford(2002)Kakade and Langford]{kakade2002approximately}
Sham Kakade and John Langford.
\newblock Approximately optimal approximate reinforcement learning.
\newblock In \emph{In Proc. 19th International Conference on Machine Learning}.
  Citeseer, 2002.

\bibitem[Kuhn(1950)]{kuhn1950simplified}
Harold~W Kuhn.
\newblock A simplified two-person poker.
\newblock \emph{Contributions to the Theory of Games}, 1:\penalty0 97--103,
  1950.

\bibitem[Kumar et~al.(2021)Kumar, Levine, and Feizi]{kumar2021policy}
Aounon Kumar, Alexander Levine, and Soheil Feizi.
\newblock Policy smoothing for provably robust reinforcement learning.
\newblock \emph{arXiv preprint arXiv:2106.11420}, 2021.

\bibitem[Lanctot et~al.(2017)Lanctot, Zambaldi, Gruslys, Lazaridou, Tuyls,
  P{\'e}rolat, Silver, and Graepel]{lanctot2017unified}
Marc Lanctot, Vinicius Zambaldi, Audrunas Gruslys, Angeliki Lazaridou, Karl
  Tuyls, Julien P{\'e}rolat, David Silver, and Thore Graepel.
\newblock A unified game-theoretic approach to multiagent reinforcement
  learning.
\newblock \emph{Advances in neural information processing systems}, 30, 2017.

\bibitem[Lanctot et~al.(2019)Lanctot, Lockhart, Lespiau, Zambaldi, Upadhyay,
  P{\'e}rolat, Srinivasan, Timbers, Tuyls, Omidshafiei,
  et~al.]{lanctot2019openspiel}
Marc Lanctot, Edward Lockhart, Jean-Baptiste Lespiau, Vinicius Zambaldi,
  Satyaki Upadhyay, Julien P{\'e}rolat, Sriram Srinivasan, Finbarr Timbers,
  Karl Tuyls, Shayegan Omidshafiei, et~al.
\newblock Openspiel: A framework for reinforcement learning in games.
\newblock \emph{arXiv preprint arXiv:1908.09453}, 2019.

\bibitem[Liang et~al.(2022)Liang, Sun, Zheng, and Huang]{liang2022efficient}
Yongyuan Liang, Yanchao Sun, Ruijie Zheng, and Furong Huang.
\newblock Efficient adversarial training without attacking: Worst-case-aware
  robust reinforcement learning.
\newblock In Alice~H. Oh, Alekh Agarwal, Danielle Belgrave, and Kyunghyun Cho
  (eds.), \emph{Advances in Neural Information Processing Systems}, 2022.
\newblock URL \url{https://openreview.net/forum?id=y-E1htoQl-n}.

\bibitem[Lin et~al.(2020)Lin, Dzeparoska, Zhang, Leon-Garcia, and
  Papernot]{lin2020robustness}
Jieyu Lin, Kristina Dzeparoska, Sai~Qian Zhang, Alberto Leon-Garcia, and
  Nicolas Papernot.
\newblock On the robustness of cooperative multi-agent reinforcement learning.
\newblock In \emph{2020 IEEE Security and Privacy Workshops (SPW)}, pp.\
  62--68. IEEE, 2020.

\bibitem[Lin et~al.(2019)Lin, Hong, Liao, Shih, Liu, and Sun]{lin2019tactics}
Yen-Chen Lin, Zhang-Wei Hong, Yuan-Hong Liao, Meng-Li Shih, Ming-Yu Liu, and
  Min Sun.
\newblock Tactics of adversarial attack on deep reinforcement learning agents,
  2019.

\bibitem[Littman(1994)]{littman1994markov}
Michael~L Littman.
\newblock Markov games as a framework for multi-agent reinforcement learning.
\newblock In \emph{Machine learning proceedings 1994}, pp.\  157--163.
  Elsevier, 1994.

\bibitem[Liu et~al.(2021)Liu, Jia, Wen, Hu, Chen, Fan, Hu, and
  Yang]{liu2021towards}
Xiangyu Liu, Hangtian Jia, Ying Wen, Yujing Hu, Yingfeng Chen, Changjie Fan,
  Zhipeng Hu, and Yaodong Yang.
\newblock Towards unifying behavioral and response diversity for open-ended
  learning in zero-sum games.
\newblock \emph{Advances in Neural Information Processing Systems},
  34:\penalty0 941--952, 2021.

\bibitem[Lockhart et~al.(2019)Lockhart, Lanctot, P{\'e}rolat, Lespiau, Morrill,
  Timbers, and Tuyls]{lockhart2019computing}
Edward Lockhart, Marc Lanctot, Julien P{\'e}rolat, Jean-Baptiste Lespiau,
  Dustin Morrill, Finbarr Timbers, and Karl Tuyls.
\newblock Computing approximate equilibria in sequential adversarial games by
  exploitability descent.
\newblock \emph{arXiv preprint arXiv:1903.05614}, 2019.

\bibitem[L{\"u}tjens et~al.(2020)L{\"u}tjens, Everett, and
  How]{lutjens2020certified}
Bj{\"o}rn L{\"u}tjens, Michael Everett, and Jonathan~P How.
\newblock Certified adversarial robustness for deep reinforcement learning.
\newblock In \emph{Conference on Robot Learning}, pp.\  1328--1337. PMLR, 2020.

\bibitem[Madry et~al.(2017)Madry, Makelov, Schmidt, Tsipras, and
  Vladu]{madry2017towards}
Aleksander Madry, Aleksandar Makelov, Ludwig Schmidt, Dimitris Tsipras, and
  Adrian Vladu.
\newblock Towards deep learning models resistant to adversarial attacks.
\newblock \emph{arXiv preprint arXiv:1706.06083}, 2017.

\bibitem[McAleer et~al.(2020)McAleer, Lanier, Fox, and
  Baldi]{mcaleer2020pipeline}
Stephen McAleer, John Lanier, Roy Fox, and Pierre Baldi.
\newblock Pipeline psro: A scalable approach for finding approximate nash
  equilibria in large games.
\newblock \emph{arXiv preprint arXiv:2006.08555}, 2020.

\bibitem[McAleer et~al.(2021)McAleer, Lanier, Wang, Baldi, and
  Fox]{mcaleer2021xdo}
Stephen McAleer, John~B Lanier, Kevin~A Wang, Pierre Baldi, and Roy Fox.
\newblock Xdo: A double oracle algorithm for extensive-form games.
\newblock \emph{Advances in Neural Information Processing Systems},
  34:\penalty0 23128--23139, 2021.

\bibitem[Mitchell et~al.(2020)Mitchell, Blumenkamp, and
  Prorok]{mitchell2020gaussian}
Rupert Mitchell, Jan Blumenkamp, and Amanda Prorok.
\newblock Gaussian process based message filtering for robust multi-agent
  cooperation in the presence of adversarial communication, 2020.

\bibitem[Mnih et~al.(2015)Mnih, Kavukcuoglu, Silver, Rusu, Veness, Bellemare,
  Graves, Riedmiller, Fidjeland, Ostrovski, et~al.]{mnih2015human}
Volodymyr Mnih, Koray Kavukcuoglu, David Silver, Andrei~A Rusu, Joel Veness,
  Marc~G Bellemare, Alex Graves, Martin Riedmiller, Andreas~K Fidjeland, Georg
  Ostrovski, et~al.
\newblock Human-level control through deep reinforcement learning.
\newblock \emph{nature}, 518\penalty0 (7540):\penalty0 529--533, 2015.

\bibitem[Mnih et~al.(2016)Mnih, Badia, Mirza, Graves, Lillicrap, Harley,
  Silver, and Kavukcuoglu]{mnih2016asynchronous}
Volodymyr Mnih, Adria~Puigdomenech Badia, Mehdi Mirza, Alex Graves, Timothy
  Lillicrap, Tim Harley, David Silver, and Koray Kavukcuoglu.
\newblock Asynchronous methods for deep reinforcement learning.
\newblock In \emph{International conference on machine learning}, pp.\
  1928--1937. PMLR, 2016.

\bibitem[Muller et~al.(2019)Muller, Omidshafiei, Rowland, Tuyls, Perolat, Liu,
  Hennes, Marris, Lanctot, Hughes, et~al.]{muller2019generalized}
Paul Muller, Shayegan Omidshafiei, Mark Rowland, Karl Tuyls, Julien Perolat,
  Siqi Liu, Daniel Hennes, Luke Marris, Marc Lanctot, Edward Hughes, et~al.
\newblock A generalized training approach for multiagent learning.
\newblock \emph{arXiv preprint arXiv:1909.12823}, 2019.

\bibitem[Munos(2003)]{munos2003error}
R{\'e}mi Munos.
\newblock Error bounds for approximate policy iteration.
\newblock In \emph{ICML}, volume~3, pp.\  560--567. Citeseer, 2003.

\bibitem[Oikarinen et~al.(2020)Oikarinen, Weng, and
  Daniel]{oikarinen2020robust}
Tuomas Oikarinen, Tsui-Wei Weng, and Luca Daniel.
\newblock Robust deep reinforcement learning through adversarial loss, 2020.

\bibitem[Pattanaik et~al.(2017)Pattanaik, Tang, Liu, Bommannan, and
  Chowdhary]{pattanaik2017robust}
Anay Pattanaik, Zhenyi Tang, Shuijing Liu, Gautham Bommannan, and Girish
  Chowdhary.
\newblock Robust deep reinforcement learning with adversarial attacks, 2017.

\bibitem[Perez-Nieves et~al.(2021)Perez-Nieves, Yang, Slumbers, Mguni, Wen, and
  Wang]{perez2021modelling}
Nicolas Perez-Nieves, Yaodong Yang, Oliver Slumbers, David~H Mguni, Ying Wen,
  and Jun Wang.
\newblock Modelling behavioural diversity for learning in open-ended games.
\newblock In \emph{International Conference on Machine Learning}, pp.\
  8514--8524. PMLR, 2021.

\bibitem[Perolat et~al.(2021)Perolat, Munos, Lespiau, Omidshafiei, Rowland,
  Ortega, Burch, Anthony, Balduzzi, De~Vylder, et~al.]{perolat2021poincare}
Julien Perolat, Remi Munos, Jean-Baptiste Lespiau, Shayegan Omidshafiei, Mark
  Rowland, Pedro Ortega, Neil Burch, Thomas Anthony, David Balduzzi, Bart
  De~Vylder, et~al.
\newblock From poincar{\'e} recurrence to convergence in imperfect information
  games: Finding equilibrium via regularization.
\newblock In \emph{International Conference on Machine Learning}, pp.\
  8525--8535. PMLR, 2021.

\bibitem[Pinto et~al.(2017)Pinto, Davidson, Sukthankar, and
  Gupta]{pinto2017robust}
Lerrel Pinto, James Davidson, Rahul Sukthankar, and Abhinav Gupta.
\newblock Robust adversarial reinforcement learning.
\newblock In \emph{International Conference on Machine Learning}, pp.\
  2817--2826. PMLR, 2017.

\bibitem[Qiaoben et~al.(2021)Qiaoben, Ying, Zhou, Su, Zhu, and
  Zhang]{qiaoben2021understanding}
You Qiaoben, Chengyang Ying, Xinning Zhou, Hang Su, Jun Zhu, and Bo~Zhang.
\newblock Understanding adversarial attacks on observations in deep
  reinforcement learning, 2021.

\bibitem[Raghunathan et~al.(2018{\natexlab{a}})Raghunathan, Steinhardt, and
  Liang]{raghunathan2018certified}
Aditi Raghunathan, Jacob Steinhardt, and Percy Liang.
\newblock Certified defenses against adversarial examples.
\newblock \emph{arXiv preprint arXiv:1801.09344}, 2018{\natexlab{a}}.

\bibitem[Raghunathan et~al.(2018{\natexlab{b}})Raghunathan, Steinhardt, and
  Liang]{raghunathan2018semidefinite}
Aditi Raghunathan, Jacob Steinhardt, and Percy~S Liang.
\newblock Semidefinite relaxations for certifying robustness to adversarial
  examples.
\newblock \emph{Advances in Neural Information Processing Systems}, 31,
  2018{\natexlab{b}}.

\bibitem[Russo \& Proutiere(2021)Russo and Proutiere]{russo2021balancing}
Alessio Russo and Alexandre Proutiere.
\newblock Balancing detectability and performance of attacks on the control
  channel of markov decision processes.
\newblock \emph{arXiv preprint arXiv:2109.07171}, 2021.

\bibitem[Schulman et~al.(2015)Schulman, Levine, Abbeel, Jordan, and
  Moritz]{schulman2015trust}
John Schulman, Sergey Levine, Pieter Abbeel, Michael Jordan, and Philipp
  Moritz.
\newblock Trust region policy optimization.
\newblock In \emph{International conference on machine learning}, pp.\
  1889--1897. PMLR, 2015.

\bibitem[Shafahi et~al.(2019)Shafahi, Najibi, Ghiasi, Xu, Dickerson, Studer,
  Davis, Taylor, and Goldstein]{shafahi2019adversarial}
Ali Shafahi, Mahyar Najibi, Mohammad~Amin Ghiasi, Zheng Xu, John Dickerson,
  Christoph Studer, Larry~S Davis, Gavin Taylor, and Tom Goldstein.
\newblock Adversarial training for free!
\newblock \emph{Advances in Neural Information Processing Systems}, 32, 2019.

\bibitem[Silver et~al.(2017)Silver, Schrittwieser, Simonyan, Antonoglou, Huang,
  Guez, Hubert, Baker, Lai, Bolton, et~al.]{silver2017mastering}
David Silver, Julian Schrittwieser, Karen Simonyan, Ioannis Antonoglou, Aja
  Huang, Arthur Guez, Thomas Hubert, Lucas Baker, Matthew Lai, Adrian Bolton,
  et~al.
\newblock Mastering the game of go without human knowledge.
\newblock \emph{nature}, 550\penalty0 (7676):\penalty0 354--359, 2017.

\bibitem[Sokota et~al.(2022)Sokota, D'Orazio, Kolter, Loizou, Lanctot,
  Mitliagkas, Brown, and Kroer]{sokota2022unified}
Samuel Sokota, Ryan D'Orazio, J~Zico Kolter, Nicolas Loizou, Marc Lanctot,
  Ioannis Mitliagkas, Noam Brown, and Christian Kroer.
\newblock A unified approach to reinforcement learning, quantal response
  equilibria, and two-player zero-sum games.
\newblock \emph{arXiv preprint arXiv:2206.05825}, 2022.

\bibitem[Srinivasan et~al.(2018)Srinivasan, Lanctot, Zambaldi, P{\'e}rolat,
  Tuyls, Munos, and Bowling]{srinivasan2018actor}
Sriram Srinivasan, Marc Lanctot, Vinicius Zambaldi, Julien P{\'e}rolat, Karl
  Tuyls, R{\'e}mi Munos, and Michael Bowling.
\newblock Actor-critic policy optimization in partially observable multiagent
  environments.
\newblock \emph{Advances in neural information processing systems}, 31, 2018.

\bibitem[Sun et~al.(2020)Sun, Zhang, Xie, Ma, Zheng, Chen, and
  Liu]{sun2020stealthy}
Jianwen Sun, Tianwei Zhang, Xiaofei Xie, Lei Ma, Yan Zheng, Kangjie Chen, and
  Yang Liu.
\newblock Stealthy and efficient adversarial attacks against deep reinforcement
  learning.
\newblock In \emph{Proceedings of the AAAI Conference on Artificial
  Intelligence}, volume~34, pp.\  5883--5891, 2020.

\bibitem[Sun et~al.(2021)Sun, Zheng, Liang, and Huang]{sun2021strongest}
Yanchao Sun, Ruijie Zheng, Yongyuan Liang, and Furong Huang.
\newblock Who is the strongest enemy? towards optimal and efficient evasion
  attacks in deep rl.
\newblock \emph{arXiv preprint arXiv:2106.05087}, 2021.

\bibitem[Sun et~al.(2022)Sun, Zheng, Hassanzadeh, Liang, Feizi, Ganesh, and
  Huang]{sun2022certifiably}
Yanchao Sun, Ruijie Zheng, Parisa Hassanzadeh, Yongyuan Liang, Soheil Feizi,
  Sumitra Ganesh, and Furong Huang.
\newblock Certifiably robust policy learning against adversarial communication
  in multi-agent systems.
\newblock \emph{arXiv preprint arXiv:2206.10158}, 2022.

\bibitem[Szegedy et~al.(2014)Szegedy, Zaremba, Sutskever, Bruna, Erhan,
  Goodfellow, and Fergus]{szegedy2014intriguing}
Christian Szegedy, Wojciech Zaremba, Ilya Sutskever, Joan Bruna, Dumitru Erhan,
  Ian~J. Goodfellow, and Rob Fergus.
\newblock Intriguing properties of neural networks.
\newblock In Yoshua Bengio and Yann LeCun (eds.), \emph{2nd International
  Conference on Learning Representations, {ICLR} 2014, Banff, AB, Canada, April
  14-16, 2014, Conference Track Proceedings}, 2014.
\newblock URL \url{http://arxiv.org/abs/1312.6199}.

\bibitem[Tessler et~al.(2019)Tessler, Efroni, and Mannor]{tessler2019action}
Chen Tessler, Yonathan Efroni, and Shie Mannor.
\newblock Action robust reinforcement learning and applications in continuous
  control.
\newblock In \emph{International Conference on Machine Learning}, pp.\
  6215--6224. PMLR, 2019.

\bibitem[Tu et~al.(2021)Tu, Wang, Wang, Manivasagam, Ren, and
  Urtasun]{tu2021adversarial}
James Tu, Tsunhsuan Wang, Jingkang Wang, Sivabalan Manivasagam, Mengye Ren, and
  Raquel Urtasun.
\newblock Adversarial attacks on multi-agent communication.
\newblock In \emph{Proceedings of the IEEE/CVF International Conference on
  Computer Vision (ICCV)}, pp.\  7768--7777, October 2021.

\bibitem[Wong \& Kolter(2018)Wong and Kolter]{wong2018provable}
Eric Wong and Zico Kolter.
\newblock Provable defenses against adversarial examples via the convex outer
  adversarial polytope.
\newblock In \emph{International Conference on Machine Learning}, pp.\
  5286--5295. PMLR, 2018.

\bibitem[Wu et~al.(2021{\natexlab{a}})Wu, Li, Huang, Vorobeychik, Zhao, and
  Li]{wu2021crop}
Fan Wu, Linyi Li, Zijian Huang, Yevgeniy Vorobeychik, Ding Zhao, and Bo~Li.
\newblock Crop: Certifying robust policies for reinforcement learning through
  functional smoothing.
\newblock \emph{arXiv preprint arXiv:2106.09292}, 2021{\natexlab{a}}.

\bibitem[Wu et~al.(2021{\natexlab{b}})Wu, Guo, Wei, and
  Xing]{wu2021adversarial}
Xian Wu, Wenbo Guo, Hua Wei, and Xinyu Xing.
\newblock Adversarial policy training against deep reinforcement learning.
\newblock In \emph{30th $\{$USENIX$\}$ Security Symposium ($\{$USENIX$\}$
  Security 21)}, 2021{\natexlab{b}}.

\bibitem[Xue et~al.(2022)Xue, Qiu, An, Rabinovich, Obraztsova, and
  Yeo]{xue2022misspoke}
Wanqi Xue, Wei Qiu, Bo~An, Zinovi Rabinovich, Svetlana Obraztsova, and
  Chai~Kiat Yeo.
\newblock Mis-spoke or mis-lead: Achieving robustness in multi-agent
  communicative reinforcement learning.
\newblock In \emph{Proceedings of the 21st International Conference on
  Autonomous Agents and Multiagent Systems}, AAMAS '22, pp.\  1418–1426,
  Richland, SC, 2022. International Foundation for Autonomous Agents and
  Multiagent Systems.
\newblock ISBN 9781450392136.

\bibitem[Zeng et~al.(2022)Zeng, Doan, and Romberg]{zeng2022regularized}
Sihan Zeng, Thinh Doan, and Justin Romberg.
\newblock Regularized gradient descent ascent for two-player zero-sum markov
  games.
\newblock \emph{Advances in Neural Information Processing Systems},
  35:\penalty0 34546--34558, 2022.

\bibitem[Zhang et~al.(2018{\natexlab{a}})Zhang, Weng, Chen, Hsieh, and
  Daniel]{zhang2018finding}
Huan Zhang, Tsui-Wei Weng, Pin-Yu Chen, Cho-Jui Hsieh, and Luca Daniel.
\newblock Efficient neural network robustness certification with general
  activation functions.
\newblock In \emph{Proceedings of the 32nd International Conference on Neural
  Information Processing Systems}, NIPS'18, pp.\  4944–4953, Red Hook, NY,
  USA, 2018{\natexlab{a}}. Curran Associates Inc.

\bibitem[Zhang et~al.(2020{\natexlab{a}})Zhang, Chen, Xiao, Gowal, Stanforth,
  Li, Boning, and Hsieh]{zhang2020towards}
Huan Zhang, Hongge Chen, Chaowei Xiao, Sven Gowal, Robert Stanforth, Bo~Li,
  Duane Boning, and Cho-Jui Hsieh.
\newblock Towards stable and efficient training of verifiably robust neural
  networks.
\newblock In \emph{International Conference on Learning Representations},
  2020{\natexlab{a}}.
\newblock URL \url{https://openreview.net/forum?id=Skxuk1rFwB}.

\bibitem[Zhang et~al.(2020{\natexlab{b}})Zhang, Chen, Xiao, Li, Boning, and
  Hsieh]{zhang2020robust}
Huan Zhang, Hongge Chen, Chaowei Xiao, Bo~Li, Duane~S Boning, and Cho-Jui
  Hsieh.
\newblock Robust deep reinforcement learning against adversarial perturbations
  on observations.
\newblock 2020{\natexlab{b}}.

\bibitem[Zhang et~al.(2021{\natexlab{a}})Zhang, Chen, Boning, and
  Hsieh]{zhang2021robust}
Huan Zhang, Hongge Chen, Duane~S Boning, and Cho-Jui Hsieh.
\newblock Robust reinforcement learning on state observations with learned
  optimal adversary.
\newblock In \emph{International Conference on Learning Representations},
  2021{\natexlab{a}}.
\newblock URL \url{https://openreview.net/forum?id=sCZbhBvqQaU}.

\bibitem[Zhang et~al.(2018{\natexlab{b}})Zhang, Yang, Liu, Zhang, and
  Basar]{zhang2018fully}
Kaiqing Zhang, Zhuoran Yang, Han Liu, Tong Zhang, and Tamer Basar.
\newblock Fully decentralized multi-agent reinforcement learning with networked
  agents.
\newblock In \emph{International Conference on Machine Learning}, pp.\
  5872--5881. PMLR, 2018{\natexlab{b}}.

\bibitem[Zhang et~al.(2021{\natexlab{b}})Zhang, Ren, and Li]{zhang2021gradient}
Runyu Zhang, Zhaolin Ren, and Na~Li.
\newblock Gradient play in stochastic games: stationary points, convergence,
  and sample complexity.
\newblock \emph{arXiv preprint arXiv:2106.00198}, 2021{\natexlab{b}}.

\end{thebibliography}
\bibliographystyle{iclr2024_conference}

\appendix
\newpage
\addcontentsline{toc}{section}{Appendix} 
\part{\Large{Appendix for ``Rethinking Adversarial Policies: A Generalized Attack Formulation and Provable Defense in RL''}} 
\parttoc
\section{Additional related work}
\label{app:related}
\textbf{Two-player zero-sum games.} The interaction between the attacker and the victim can be modeled as a two-player zero-sum game. There is a large body of work using RL to solve Nash equilibrium (NE). For example, \cite{sokota2022unified, perolat2021poincare,mcaleer2021xdo,brown2019deep,lockhart2019computing} considers the normal-form and extensive-form games. \cite{heinrich2016deep, lanctot2017unified, mcaleer2020pipeline,perez2021modelling,liu2021towards} deal with more general two-player zero-sum games and propose population-based RL algorithms showing empirical success but lacking provable finite-time guarantees. \cite{daskalakis2020independent,zeng2022regularized} analyze the theory of independent gradient with different learning rates in {\it Markov game}, which is a {\it special case} of our defense problem against {\it unconstrained attack} and serves as the inspiration for us to develop provable adversarial training algorithms.

\textbf{Adversarial attacks on RL policies.}
As deep neural networks are shown to be vulnerable to adversarial attacks~\citep{szegedy2014intriguing,goodfellow2015explaining}, the adversarial robustness of deep RL policies has also attracted increasing attention.
One of the earliest works by \cite{huang2017adversarial} reveals the vulnerability of neural policies by adapting various adversarial attacks from supervised learning to RL policies. 
\cite{lin2019tactics} consider an efficient attack that only perturbs the agent at a subset of time steps.
There has been subsequent research in developing stronger pixel-based attacks \citep{qiaoben2021understanding, pattanaik2017robust, oikarinen2020robust}. 
\cite{zhang2020robust} built the theoretical framework SA-MDP for adversarial state perturbation and proposed the corresponding regularizer for more robust reinforcement learning policies. 
Subsequent work by \cite{sun2021strongest} improves \cite{zhang2020robust} with the framework of PA-MDP for better efficiency.
The majority of related work on adversarial RL focuses on perturbing state observations~\citep{huang2017adversarial,oikarinen2020robust,sun2021strongest}, and assumes that the perturbation is small in $\ell_p$ distance. In contrast, our paper considers the attack generated by other agents in a multi-agent system and does not restrict the perturbation distance in every single step, allowing for more flexible and practical attack models.

\textbf{Adversarial attacks on multi-agent RL (MARL).}
\cite{gleave2019adversarial} investigate adversarial policies in a two-player zero-sum game, where a victim can be exploited and significantly misled by the opponent's changed behavior. 
\cite{guo2021adversarial} remove the zero-sum assumption in \cite{gleave2019adversarial} and construct a new formulation for the adversarial policy.
\cite{lin2020robustness} study adversarial attacks in cooperative MARL systems and reveal that attacking one of the agents in the team can greatly reduce the total reward. 
However, these adversarial policies are unconstrained and could cause abnormal behaviors that are easily detectable, while our attack model can be made stealthy by restricting the state distribution shifts.
We also propose provably robust defense algorithms to learn a policy that is unexploitable. 

\textbf{Attacks and defenses on communication in MARL.}
There is also a line of work studying attacks and defenses on communications in MARL~\citep{blumenkamp2020emergence,tu2021adversarial,mitchell2020gaussian,xue2022misspoke}, where the communication among cooperative agents could be perturbed to influence the decisions of victims.
However, we consider adversarial behaviors or policies of other agents, which affect the victim by the actions taken by other agents. 

\textbf{Provably robust defenses.}
To provide guaranteed robustness for deep neural networks, many approaches have been developed to certify the performance of neural networks, including semidefinite programming-based defenses~\citep{raghunathan2018certified,raghunathan2018semidefinite}, convex relaxation of neural networks~\citep{gowal2019scalable,zhang2018finding,wong2018provable,zhang2020towards,gowal2018effectiveness}, randomized smoothing of a classifier~\citep{cohen2019certified,hayes2020extensions}, etc. 
In an effort to certify RL agents' robustness, some approaches~\citep{lutjens2020certified,zhang2020robust,oikarinen2020robust,fischer2019online} apply network certification tools to bound the Q networks. \cite{kumar2021policy} and \cite{wu2021crop} apply randomized smoothing~\citep{cohen2019certified} to RL to achieve provable robustness. These 
defenses mainly focus on the adversarial perturbations directly applied to the agent's inputs.
\cite{sun2022certifiably} propose a certifiable defense against adversarial communication in MARL systems. 
To the best of our knowledge, our paper is the first to provide provable convergence guarantees for adversarial training against adversarial attacks on the behaviors of other agents in the environment.

\textbf{Improving policy robustness by adversarial training.} 
Prior work shows that the competition between the victim agent and the adversary can be regarded as a two-player zero-sum game while training agents with learned adversarial attacks can improve the robustness of the victim. Such an adversarial training paradigm has been shown effective under state perturbations~\citep{zhang2021robust,sun2021strongest,liang2022efficient} and action perturbations~\citep{pinto2017robust,tessler2019action} on a single victim. 
In the context of perturbing the actions of the victim, \cite{tessler2019action} presents a two-player policy iteration algorithm that is proved to converge to the Nash Equilibrium. In contrast, we consider the adversarial behaviors of other agents, and we provide both theoretical guarantees and empirical evidence for the effectiveness of our adversarial training. 

\section{Relationship between NE and robustness}\label{sec:ne_robust}
To understand the relationship between NE and robustness, we formally define the NE as follows.
\begin{definition}[Nash equilibrium]\label{def:ne}
	Fix $\hat{\pi}_\alpha$ and $\epsilon_\pi\ge 0$. We say a pair of policy $(\pi_\nu^\star, \pi_\alpha^\star)$ the Nash equilibrium for the zero-sum game between the victim and the attacker if it holds that for any $\pi_\nu^\prime$ and $\pi_\alpha^\prime$:
	\begin{align}
		\label{eq:ne1}V_\rho(\pi_\nu^\star, (1-\epsilon_\pi)\hat{\pi}_\alpha + \epsilon_\pi \pi_\alpha^\star)\ge V_\rho(\pi_\nu^\prime, (1-\epsilon_\pi)\hat{\pi}_\alpha + \epsilon_\pi\pi_\alpha^\star),\\
		\label{eq:ne2}V_\rho(\pi_\nu^\star, (1-\epsilon_\pi)\hat{\pi}_\alpha + \epsilon_\pi\pi_\alpha^\star)\le V_\rho(\pi_\nu^\star, (1-\epsilon_\pi)\hat{\pi}_\alpha + \epsilon_\pi\pi_\alpha^\prime).
		\end{align}
\end{definition}
Now we have the following proposition showing the robustness of the NE.
\begin{proposition}
	Fix $\hat{\pi}_\alpha$ and $\epsilon_\pi\in [0, 1]$. If $(\pi_\nu^\star, \pi_\alpha^\star)$ is the Nash equilibrium for the zero-sum game between the victim and the attacker, then $\pi_\nu^\star$ is the minimizer for the function $\operatorname{Expl}(\pi_\nu^\prime)$.
\end{proposition}
\begin{proof}
		According to Definition \ref{def:ne}, it holds that
	\begin{align}\label{eq:minmax}
		\min_{\pi_\alpha^\prime}V_{\rho}(\pi_{\nu}^\star, (1-\epsilon_\pi)\hat{\pi}_\alpha + \epsilon_\pi\pi_{\alpha}^\prime)=V_{\rho}(\pi_{\nu}^\star, (1-\epsilon_\pi)\hat{\pi}_\alpha + \epsilon_\pi\pi_{\alpha}^\star) = \max_{\pi_\nu^\prime} V_{\rho}(\pi_{\nu}^\prime, (1-\epsilon_\pi)\hat{\pi}_\alpha + \epsilon_\pi\pi_{\alpha}^\star).
	\end{align}
	Combining Equation \eqref{eq:minmax} with the fact that for any function $f$ it holds that $\min_x\max_y f(x, y)\ge \max_y\min_x f(x, y)$, we conclude that
	\begin{align*}
		V_{\rho}(\pi_{\nu}^\star, (1-\epsilon_\pi)\hat{\pi}_\alpha + \epsilon_\pi\pi_{\alpha}^\star) &= \max_{\pi_\nu^\prime}\min_{\pi_\alpha^\prime} V_\rho (\pi_\nu^\prime, (1-\epsilon_\pi)\hat{\pi}_\alpha + \epsilon_\pi\pi_{\alpha}^\prime)\\
			&= \min_{\pi_\alpha^\prime} \max_{\pi_\nu^\prime} V_\rho (\pi_\nu^\prime, (1-\epsilon_\pi)\hat{\pi}_\alpha + \epsilon_\pi\pi_{\alpha}^\prime).
	\end{align*}
	Now given an NE pair $(\pi_\nu^\star, \pi_\alpha^\star)$, we consider the  exploitability of $\pi_\nu^\star$:		\begin{align*}
		\operatorname{Expl}(\pi_\nu^\star) &= -\min_{\pi_\alpha} V_{\rho}(\pi_{\nu}^\star, (1-\epsilon_\pi)\hat{\pi}_\alpha + \epsilon_\pi\pi_{\alpha})\\
		&=-V_{\rho}(\pi_{\nu}^\star, (1-\epsilon_\pi)\hat{\pi}_\alpha + \epsilon_\pi\pi_{\alpha}^\star)\\
		&= \min_{\pi_\nu^\prime}-\min_{\pi_\alpha^\prime} V_\rho (\pi_\nu^\prime, (1-\epsilon_\pi)\hat{\pi}_\alpha + \epsilon_\pi\pi_{\alpha}^\prime)\\
		&= \min_{\pi_\nu^\prime} \operatorname{Expl}(\pi_\nu^\prime).
	\end{align*}
\end{proof}

Therefore, the solution concept NE guarantees the most robust policy for the victim. Indeed, it is stronger than the goal of the most robust policy since by following exactly the same procedure, one can verify NE also guarantees the most robust policy for the {\it attacker}. However, as we remarked previously, we only care about the robustness of the victim.
\section{Motivation and examples of timescale separation}
\label{sec:motivation}

To motivate the necessity of timescale separation in adversarial training for robust RL policy, we revisit some known issues of naive single-timescale methods including Gradient Descent Ascent (GDA) and iterative best response (IBR) with both simultaneous and alternating update using a simple normal-form game Rock-Paper-Scissor, which corresponds to our defense problem with a single state and $\epsilon_\pi = 1$.
\begin{example}
The zero-sum game Rock-Paper-Scissor includes two players with the same action space $\mathcal{A} = \{Rock, Paper, Scissor\}$. The payoff matrix $\mathbf{P}$ is given as $\mathbf{P}=\left[\begin{array}{ccc}0 & 1 & -1 \\ -1 & 0 & 1 \\ 1 & -1 & 0\end{array}\right]$ for the row player. The row player has the mixed strategy $x\in \mathcal{X}=\Delta(\mathcal{A})$, where $x_i$ represents the probability of choosing $i^{th}$ action. The column player holds a similar mixed strategy $y\in \mathcal{Y}$. The corresponding payoff is given by $V(x, y) = x^{\top}\mathbf{P}y$. Our objective is given by $\max_{x\in \mathcal{X}}\min_{y\in \mathcal{Y}}x^{\top}\mathbf{P}y$. Note here ${\epsilon_\pi} = 1$.
\end{example}

Formally, we compare the following $5$ methods; the first $4$ methods have only a single timescale while the last GAMin method highlights timescale separation, where one player takes a gradient step, and then another one takes the best response.
\begin{itemize}
\setlength\itemsep{0.01em}
    \item Simultaneous gradient descent ascent (SGDA): $y_{t+1} = \mathcal{P}_{\mathcal{Y}}(y_t - \eta \mathbf{P}^{\top}x_t)$, $x_{t+1} = \mathcal{P}_{\mathcal{X}}(x_t + \eta \mathbf{P}y_t)$.
    \item Alternate gradient descent ascent (AGDA): $y_{t+1} = \mathcal{P}_{\mathcal{Y}}(y_t - \eta \mathbf{P}^{\top}x_t)$, $x_{t+1} = \mathcal{P}_{\mathcal{X}}(x_t + \eta \mathbf{P}y_{t+1})$
    \item Simultaneous iterative best response (SIBR): $y_{t+1} = \arg\min_{y\in\mathcal{Y}} x_{t}^{\top}\mathbf{P}y$, $x_{t+1} = \arg\min_{x\in\mathcal{X}} x^{\top}\mathbf{P}y_t$
    \item Alternate iterative best response (AIBR): $y_{t+1} = \arg\min_{y\in\mathcal{Y}} x_{t}^{\top}\mathbf{P}y$, $x_{t+1} = \arg\min_{x\in\mathcal{X}} x^{\top}\mathbf{P}y_{t+1}$.
    \item (With timescale separation) Gradient ascent with min oracle (GAMin): 
     $y_{t+1} = \arg\min_{y\in\mathcal{Y}} x_{t}^{\top}\mathbf{P}y$, $x_{t+1} = \mathcal{P}_{\mathcal{X}}(x_t + \eta \mathbf{P}y_{t+1})$.
\end{itemize}
%

\begin{figure}[!htbp]
   \centering
   \includegraphics[width=0.5\columnwidth]{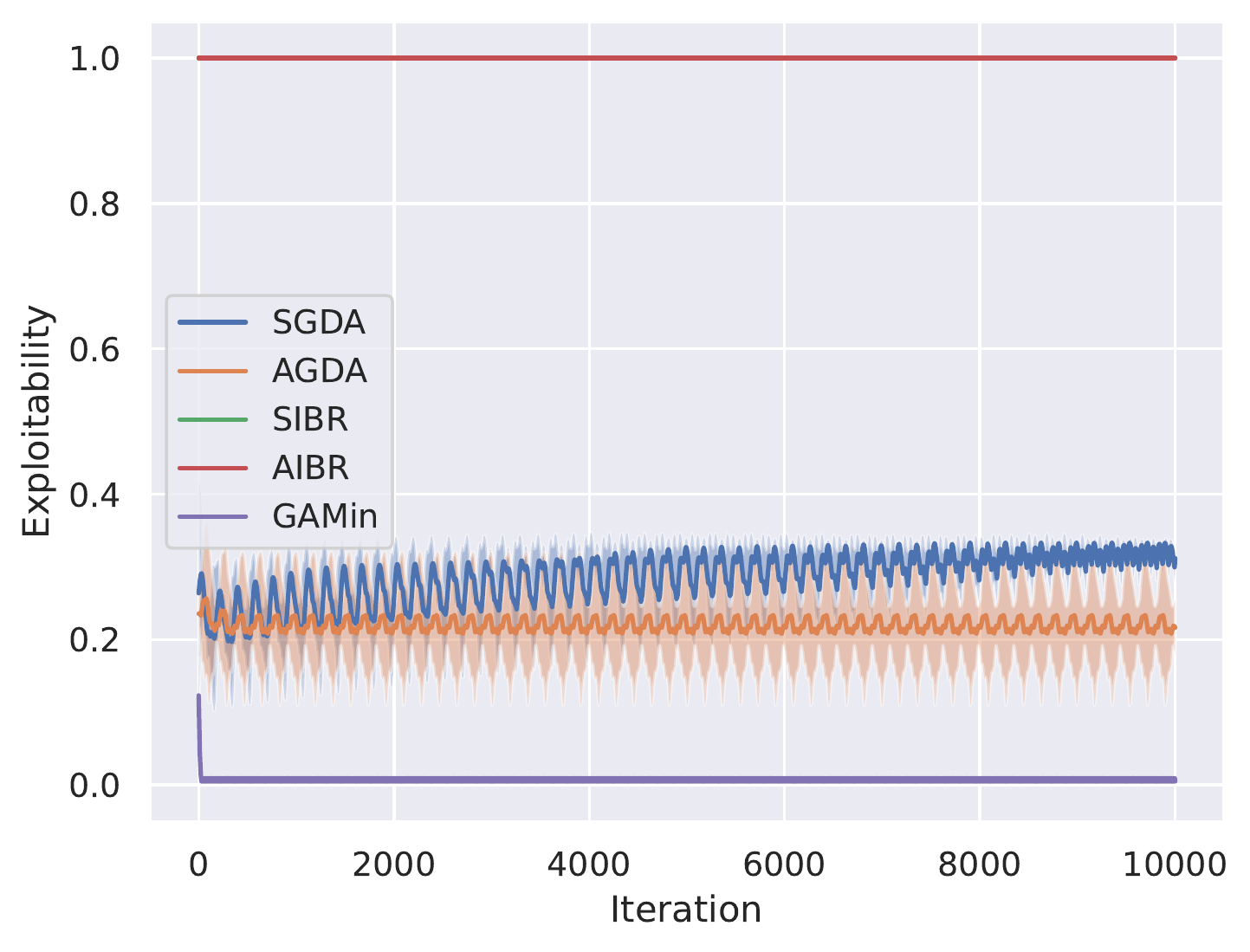}
   \caption{Exploitability test on Rock-Paper-Scissor. Note that the green line is overlapped with the red. 
   } 
   \label{fig:toy}
\end{figure}

We show the exploitability of the $x$ player during the learning process in Figure \ref{fig:toy}. It is clear that the first four single-timescale training methods (SGDA, AGDA, SIBR, AIBR)
fail to achieve low exploitability, while only GAMin achieves the near-optimal exploitability $0$. 

%

\section{Full proof}\label{sec:full_proof}

\subsection{Proof of Proposition~\ref{theorem:1} and \ref{theorem:2_new}}\label{app:1}
\begin{proof}
For simplicity, we shall prove Proposition \ref{theorem:2_new} first, and then prove Proposition \ref{theorem:1}.
 
Let us first review the following facts for any joint policy $\bm{\pi} = (\pi_\nu, \pi_\alpha)$, $\bm{\pi}^{\prime} = (\pi_\nu, \pi_\alpha^{\prime})$ such that $D_{\operatorname{TV}}^{\max}(\pi_{\alpha}||\pi_{\alpha}^{\prime})\le \epsilon_\pi$ and the transition matrix $P_{\bm{\pi}}$, where $P_{\bm{\pi}}(s^{\prime}, s) = \sum_{\bm{a}}\bm{\pi}(\bm{a}|s)P(s^{\prime}|s, \bm{a})$. In the following proof, we use $P_{\bm{\pi}}(i, j)$ to denote $P_{\bm{\pi}}(s_i, s_j)$.
\begin{itemize}
\item $d_\rho^{\bm{\pi}} = (1-\gamma)(I-\gamma P_{\bm{\pi}})^{-1}\rho$.
\item $||P_{\bm{\pi}}||_1=1$ and $||(I-\gamma P_{\bm{\pi}})^{-1}||_1 \le \frac{1}{1-\gamma}$.
\item $||P_{\bm{\pi}}-P_{\bm{\pi}^{\prime}}||_1 \le 2\epsilon_\pi$
\end{itemize}
According to Definition \ref{def:state_dis}, one can verify that $d_\rho^{\bm{\pi}}$ satisfies that:
\begin{align*}
    d_\rho^{\bm{\pi}} = (1-\gamma)\rho + \gamma P_{\bm{\pi}}d_\rho^{\bm{\pi}},
\end{align*}
which gives the solution $d_\rho^{\bm{\pi}} = (1-\gamma)(I-\gamma P_{\bm{\pi}})^{-1}\rho$.

For $P_{\bm{\pi}}$:
\begin{align*}
    ||P&_{\bm{\pi}}||_{1}
    =\max_{j}\sum_{i}|P_{\bm{\pi}}(i, j)|\\
    =&\max_{j}\sum_{i}\sum_{\bm{a}}\bm{\pi}(\bm{a}|s_j)P(s_i|s_j, \bm{a})\\
    =&\max_{j}\sum_{\bm{a}}\bm{\pi}(\bm{a}|s_j)\sum_{i}P(s_i|s_j, \bm{a})\\
    =&1
\end{align*}
For $||(I-\gamma P_{\bm{\pi}})^{-1}||_1$:
\begin{align*}
    ||&(I-\gamma P_{\bm{\pi}})^{-1}||_1 = || \sum_{k=0}^{\infty}(\gamma P_{\bm{\pi}})^k||_1\le \sum_{k=1}^{\infty}|| (\gamma P_{\bm{\pi}})^k ||_1\le
    \sum_{k=1}^{\infty}\gamma^k||P_{\bm{\pi}}||_1^k = \frac{1}{1-\gamma}.
\end{align*}
For $||P_{\bm{\pi}}-P_{\bm{\pi}^{\prime}}||_1$:

To begin with, since $\pi_\nu$ remains unchanged in our proof, let us abuse the notation a little bit and define the marginalized transition $P(s^{\prime}|s, a_{\alpha}) = \sum_{a_\nu}\pi_{\nu}(a_\nu|s)P(s^\prime|s, a_\nu, a_\alpha)$. We have
\begin{align*}
    ||&P_{\bm{\pi}}-P_{\bm{\pi}^{\prime}}||_1=\max_j\sum_i|P_{\bm{\pi}}(i, j)-P_{\bm{\pi}^{\prime}}(i, j)|\\
    &=\max_j\sum_i|\sum_{a_\alpha}(\pi_{\alpha}(a_\alpha|s_j) - \pi_{\alpha}^{\prime}(a_\alpha|s_j)\sum_{a_\nu}\pi_{\nu}(a_\nu|s_j)P(s_i|s_j, a_\nu, a_\alpha)|\\
    &=\max_j\sum_i|\sum_{a_\alpha}(\pi_{\alpha}(a_\alpha|s_j) - \pi_{\alpha}^{\prime}(a_\alpha|s_j))P(s_i|s_j, a_{\alpha})|.
\end{align*}
Now fix any index $j$, define $\bm{m}_i^{\top} =(P(s_i|s_j, a_\alpha^{k}))_{k=1}^{|\mathcal{A}_\alpha|}$, $\bm{M}^{\top} = (\bm{m}_1, \cdots, \bm{m}_{|\mathcal{S}|})$, and $\bm{n}^{\top} = (\pi_{\alpha}(a_\alpha^k|s_j) - \pi_{\alpha}^{\prime}(a_\alpha^k|s_j))_{k=1}^{|\mathcal{A}_\alpha|}$. Then the following holds
\begin{align*}
    \sum_i|\sum_{a_\alpha}(\pi_{\alpha}(a_\alpha|s_j) - \pi_{\alpha}^{\prime}(a_\alpha|s_j))P(s_i|s_j, a_{\alpha})| = \sum_i|\bm{m}_i^{\top}\bm{n}|=||\bm{M}\bm{n}||_1\le ||\bm{M}||_1||\bm{n}||_1=2\epsilon_\pi||\bm{M}||_1.
\end{align*}
According to the definition of $\bm{M}$, it is easy to check
\begin{align*}
    ||\bm{M}||_1 = \max_{k}\sum_i |P(s_i|s_j, a_\alpha^{k})| = 1.
\end{align*}
Therefore, we conclude that for any fixed index $j$, we have
\begin{align*}
    \sum_i|\sum_{a_\alpha}(\pi_{\alpha}(a_\alpha|s_j) - \pi_{\alpha}^{\prime}(a_\alpha|s_j))P(s_i|s_j, a_{\alpha})|\le 2\epsilon_\pi,
\end{align*}
which proves $||P_{\bm{\pi}}-P_{\bm{\pi}^{\prime}}||_1\le 2\epsilon_\pi$.

Now we are ready to prove Proposition \ref{theorem:2_new}.
\begin{align*}
    ||&d_\rho^{\bm{\pi}}-d_\rho^{\bm{\pi}^{\prime}}||_1 = ||(1-\gamma)(I-\gamma P_{\bm{\pi}})^{-1}\rho - (1-\gamma)(I-\gamma P_{\bm{\pi}^{\prime}})^{-1}\rho||_1\\
    &\le (1-\gamma)|| (I-\gamma P_{\bm{\pi}})^{-1} - (I-\gamma P_{\bm{\pi}^{\prime}})^{-1} ||_1||\rho||_1\\
    &\le (1-\gamma)||(I-\gamma P_{\bm{\pi}^{\prime}})^{-1}||_1||\gamma(P_{\bm{\pi}} - P_{\bm{\pi}^{\prime}})||_1||(I-\gamma P_{\bm{\pi}})^{-1}||_1||\rho||_1\\
    &\le \frac{2\epsilon_\pi \gamma}{1-\gamma}.
\end{align*}
Now we can use Proposition \ref{theorem:2_new} to prove Proposition \ref{theorem:1}. To begin with, it is easy to verify that the following holds
\begin{align*}
    V_{\rho}(\bm{\pi}) = \sum_{s}d_\rho^{\bm{\pi}}(s)\sum_{a}\bm{\pi}(\bm{a}|s)r(s, \bm{a}).
\end{align*}
Now let us define the marginalized reward $r_{\bm{\pi}}(s) = \sum_{a}\bm{\pi}(\bm{a}|s)r(s, \bm{a})$, and further define the vector notation $r_{\bm{\pi}}^{\top} = (r_{\bm{\pi}}(s^k))_{k=1}^{|\mathcal{S}|}$. Then for the difference of the value function, it holds that
\begin{align*}
    |&V_{\rho}(\bm{\pi}) - V_{\rho}(\bm{\pi}^{\prime})| = \frac{1}{1-\gamma}|\langle d_\rho^{\bm{\pi}}, r_{\bm{\pi}}\rangle - \langle d_\rho^{\bm{\pi}^{\prime}}, r_{\bm{\pi}^{\prime}}\rangle|\\
    &=\frac{1}{1-\gamma}|\langle d_\rho^{\bm{\pi}}, r_{\bm{\pi}}\rangle -\langle d_\rho^{\bm{\pi}}, r_{\bm{\pi}^{\prime}}\rangle + \langle d_\rho^{\bm{\pi}}, r_{\bm{\pi}^{\prime}}\rangle - \langle d_\rho^{\bm{\pi}^{\prime}}, r_{\bm{\pi}^{\prime}}\rangle|\\
    &\le \frac{1}{1-\gamma}(|\langle d_\rho^{\bm{\pi}}, r_{\bm{\pi}}\rangle -\langle d_\rho^{\bm{\pi}}, r_{\bm{\pi}^{\prime}}\rangle| + |\langle d_\rho^{\bm{\pi}}, r_{\bm{\pi}^{\prime}}\rangle - \langle d_\rho^{\bm{\pi}^{\prime}}, r_{\bm{\pi}^{\prime}}\rangle|)\\
    &\le \frac{1}{1-\gamma}(||d_\rho^{\bm{\pi}}||_1||r_{\bm{\pi}} - r_{\bm{\pi}^\prime} ||_{\infty} + ||d_\rho^{\bm{\pi}} - d_\rho^{\bm{\pi}^\prime}||_1||r_{\bm{\pi}^\prime}||_{\infty})\\
    &\le \frac{1}{1-\gamma}(2\epsilon_\pi + \frac{2\epsilon_\pi\gamma}{1-\gamma})\\
    &\le \frac{2\epsilon_\pi}{(1-\gamma)^2}.
\end{align*}
\end{proof}
\subsection{Proof of Proposition \ref{prop:detection}}
\begin{proof}
	For any fixed $s\in\cS$ and $a_\nu\in\cA_\nu$, the channel that produces the distribution of next state $s^\prime$ based on the input distribution of $a_\alpha$ is exactly $P(\cdot\given s, a_\nu, \cdot)$. Therefore, applying the data processing inequality for $f$-divergence, we prove our proposition.
\end{proof}

\subsection{Proof of Theorem \ref{thm:max}}
In the following discussions, we will use $J_{\epsilon_\pi}(\nu, \alpha)$ and $J_{\epsilon_\pi}(\pi_\nu, \pi_\alpha)$ interchangeably according to Definition \ref{def:param}. Once we have a bounded mismatch coefficient in Definition \ref{assump:c_g}, we are ready to analyze the properties of the function $J_{\epsilon_\pi}(\nu, \alpha)$ in the following lemma.

\begin{lemma}\label{lemma:three}
For any $\nu, \nu^\prime\in \Delta(\cA_\nu)^{|\cS|}$ and $\alpha, \alpha^\prime\in \Delta(\cA_\alpha)^{|\cS|}$,
the function $J_{\epsilon_\pi}$ satisfies 
\begin{itemize}
    \item 
    Lipschitzness 
    \begin{align*}
        \|\nabla_{\nu} J_{\epsilon_\pi}(\nu, \alpha)\|\le \frac{\sqrt{|\cA_\nu|}}{(1-\gamma)^2},\\
        \|\nabla_{\alpha} J_{\epsilon_\pi}(\nu, \alpha)\|\le \frac{\epsilon_\pi\sqrt{|\cA_\alpha|}}{(1-\gamma)^2}.
    \end{align*}
    \item 
    Smoothness
    \begin{align*}
        \| \nabla_{\nu} J_{\epsilon_\pi}(\nu, \alpha) -  \nabla_{\nu} J_{\epsilon_\pi}(\nu^\prime, \alpha^\prime)\|
        &\le \frac{2\sqrt{|\cA_\nu|}}{(1-\gamma)^3}(\sqrt{|\cA_\nu|}\|\nu-\nu^\prime\| + \sqrt{|\cA_\alpha|}\|\alpha-\alpha^\prime\|),\\
        \| \nabla_{\alpha} J_{\epsilon_\pi}(\nu, \alpha) -  \nabla_{\alpha} J_{\epsilon_\pi}(\nu^\prime, \alpha^\prime)\|
        &\le \frac{2{\epsilon_\pi}\sqrt{|\cA_\alpha|}}{(1-\gamma)^3}(\sqrt{|\cA_\nu|}\|\nu-\nu^\prime\| + \sqrt{|\cA_\alpha|}\|\alpha-\alpha^\prime\|).
    \end{align*}
    \item 
    Gradient domination 
    \begin{align}
        &J_{\epsilon_\pi}(\nu, \alpha) -\min_{\alpha^\prime}J_{\epsilon_\pi}(\nu, \alpha^\prime)\le\label{eq:grad_alpha}
        \frac{C_{\cG}^{\epsilon_\pi}}{1-\gamma}\max_{\Bar{\alpha}}\langle \nabla_{\alpha}J_{\epsilon_\pi}(\nu, \alpha), \alpha-\Bar{\alpha}\rangle,
        \\
        &\max_{\nu^\prime}J_{\epsilon_\pi}(\nu^\prime, \alpha) -J_{\epsilon_\pi}(\nu, \alpha)\le\label{eq:grad_nu}
        \frac{C_{\cG}^{\epsilon_\pi}}{1-\gamma}\max_{\Bar{\nu}}\langle \nabla_{\nu}J_{\epsilon_\pi}(\nu, \alpha), \Bar{\nu}-\nu\rangle.
    \end{align}  
\end{itemize}
\end{lemma}

\begin{proof}
    With the definition of the function $J_{\epsilon_\pi}$, we have $J_{{\epsilon_\pi}}(\nu, \alpha)=V_{\rho}(\pi_\nu, (1-{\epsilon_\pi})\hat{\pi}_\alpha +{\epsilon_\pi}\pi_\alpha)$. 
    \begin{itemize}
        \item For Lipschitzness, it is easy to compute of gradient of $J_{\epsilon_\pi}$ with respect to $\nu$ and $\alpha$ using chain rules. Let's denote $\pi^{mix}_\alpha = (1-{\epsilon_\pi})\hat{\pi}_\alpha + {\epsilon_\pi}\pi_{\alpha}$. Then the following holds using chain rules and standard policy gradient expression \cite{zhang2021gradient}.
    \begin{align*}
        \frac{\partial J_{\epsilon_\pi}(\nu, \alpha)}{\partial \nu_{s, a_\nu}} = \frac{\partial V_{\rho}(\pi_\nu, \pi^{mix}_\alpha)}{\partial \nu_{s, a}} = \frac{1}{1-\gamma}d_\rho^{\pi_\nu, \pi^{mix}_\alpha}(s) \EE_{a_{\alpha}\sim \pi^{mix}_\alpha}[Q^{\pi_\nu, \pi^{mix}_\alpha}(s, a_{\nu}, a_{\alpha})]\le \frac{d_\rho^{\pi_\nu, \pi^{mix}_\alpha}(s)}{(1-\gamma)^2},\\
        \frac{\partial J_{\epsilon_\pi}(\nu, \alpha)}{\partial \alpha_{s, a_\alpha}} = \frac{\partial V_{\rho}(\pi_\nu, \pi^{mix}_\alpha)}{\partial \alpha_{s, a_\alpha}} = \frac{{\epsilon_\pi}}{1-\gamma}d_\rho^{\pi_\nu, \pi^{mix}_\alpha}(s) \EE_{a_{\nu}\sim {\pi}_\nu}[Q^{\pi_\nu, \pi^{mix}_\alpha}(s, a_{\nu}, a_{\alpha})]\le \frac{\epsilon_\pi d_\rho^{\pi_\nu, \pi^{mix}_\alpha}(s)}{(1-\gamma)^2}.
    \end{align*}
Therefore, we have $\|\nabla_{\nu} J_{\epsilon_\pi}(\nu, \alpha)\|\le \frac{\sqrt{|\cA_\nu|}}{(1-\gamma)^2}, \|\nabla_{\alpha} J_{\epsilon_\pi}(\nu, \alpha)\|\le \frac{{\epsilon_\pi}\sqrt{|\cA_\alpha|}}{(1-\gamma)^2}$.

    \item For Smoothness, it holds that
\begin{align*}
    &\| \nabla_{\nu} J_{\epsilon_\pi}(\nu, \alpha) -  \nabla_{\nu} J_{\epsilon_\pi}(\nu^\prime, \alpha^\prime)\| \\
    &= \| \nabla_{\pi_\nu} V(\pi_\nu, (1-{\epsilon_\pi})\hat{\pi}_\alpha + {\epsilon_\pi}\pi_{\alpha}) -  \nabla_{\pi_\nu} V(\pi_\nu^\prime, (1-{\epsilon_\pi})\hat{\pi}_\alpha + {\epsilon_\pi}\pi_{\alpha}^\prime)\|\\
    &\le \frac{2\sqrt{|\cA_\nu|}}{(1-\gamma)^3}(\sqrt{|\cA_\nu|}\|\nu-\nu^\prime\| + \sqrt{|\cA_\alpha|}\|\alpha-\alpha^\prime\|).
\end{align*}
where the last step comes from Lemma $19$ in \citep{zhang2021gradient}. Similarly, using chain rules, it holds that
\begin{align*}
    &\| \nabla_{\alpha} J_{\epsilon_\pi}(\nu, \alpha) -  \nabla_{\alpha} J_{\epsilon_\pi}(\nu^\prime, \alpha^\prime)\| \\
    &= {\epsilon_\pi}\| \nabla_{\pi_\alpha^{mix}} V(\pi_\nu, \pi_\alpha^{mix}) \bigl\vert_{\pi_\alpha^{mix} = (1-{\epsilon_\pi})\hat{\pi}_\alpha + {\epsilon_\pi}\pi_{\alpha}} -  \nabla_{\pi_\alpha^{mix}} V(\pi_\nu^\prime, {\pi}_\alpha^{mix})\bigl\vert_{\pi_\alpha^{mix} = (1-{\epsilon_\pi})\hat{\pi}_\alpha + {\epsilon_\pi}\pi_{\alpha}^\prime}\|\\
    &\le \frac{2{\epsilon_\pi}\sqrt{|\cA_\alpha|}}{(1-\gamma)^3}(\sqrt{|\cA_\nu|}\|\nu-\nu^\prime\| + \sqrt{|\cA_\alpha|}\|\alpha-\alpha^\prime\|).
\end{align*}
    \item The most challenging step is establishing the gradient domination properties. We will show that our objective $J_{\epsilon_\pi}(\pi_\nu, \pi_\alpha)$ is indeed a value function of some other Markov games $\Tilde{\cG}$ with modified transition and reward compared with the original ${\cG}$, which is defined as for any given $s, a_\nu, a_\alpha, s^\prime$,
    \begin{align*}
        {r}_{\nu}^{mix}(s, a_{\nu}, a_\alpha) &= (1-{\epsilon_\pi})\sum_{a_\alpha^\prime}r_\nu(s, a_\nu, a_\alpha^\prime)\hat{\pi}_\alpha(a_\alpha^\prime\given s) + {\epsilon_\pi} r_{\nu}(s, a_\nu, a_\alpha),\\
        {P}^{mix}(s^\prime|s, a_\nu, a_\alpha) &= (1-{\epsilon_\pi})\sum_{a_{\alpha}^\prime }P(s^\prime\given s, a_\nu, a_\alpha^\prime)\hat{\pi}_\alpha(a_\alpha^\prime\given s) + {\epsilon_\pi} P(s^\prime \given s, a_\nu, a_\alpha).
    \end{align*}
    For this modified Markov game, we denote the corresponding value function as $\Tilde{V}_\rho(\pi_\nu, \pi_\alpha)$, and the stationary state visitation as $\Tilde{d}_{\rho}^{\pi_\nu, \pi_\alpha}$. Then it is easy to verify that for any $\pi_\nu\in\Pi_{\nu}$, $\pi_\alpha\in\Pi_{\alpha}$
    \begin{align}
        &J_{\epsilon_\pi}(\pi_\nu, \pi_\alpha) = \Tilde{V}_\rho (\pi_\nu, \pi_\alpha),\label{eq:V}\\
        &d_\rho^{\pi_\nu, (1-{\epsilon_\pi})\hat{\pi}_\alpha + {\epsilon_\pi} \pi_\alpha}=\Tilde{d}_\rho^{\pi_\nu, \pi_\alpha}\label{eq:d}.
    \end{align}
    Similarly, we define the mismatch coefficient for $\Tilde{\cG}$ as follows:
    \begin{align*}
    C_{\Tilde{\cG}}:=\max \{&\max_{\pi_\nu\in {\Pi}_{\nu}} \min_{\pi_\alpha \in \Tilde{\Pi}_\alpha^{\star}\left(\pi_\nu\right)}\left\|\frac{\Tilde{d}_\rho^{\pi_\nu, {\pi}_\alpha}}{\rho}\right\|_{\infty},\max_{\pi_\alpha\in\Pi_\alpha} \min_{\pi_\nu \in \Tilde{\Pi}_\nu^{\star}\left(\pi_\alpha\right)}\left\|\frac{\Tilde{d}_\rho^{\pi_\nu, \pi_\alpha}}{\rho}\right\|_{\infty}\},
    \end{align*}
    where $\Tilde{\Pi}_\alpha^{\star}\left(\pi_\nu\right):=\arg\min_{\pi_\alpha\in \Pi_{\alpha}} \Tilde{V}_\rho(\pi_\nu, \pi_\alpha)$, and $\Tilde{\Pi}_\nu^{\star}\left(\pi_\alpha\right):=\arg\max_{\pi_\nu\in \Pi_{\nu}} \Tilde{V}_\rho(\pi_\nu, \pi_\alpha)$.
    Next, we use the gradient domination property of the value function of Lemma 3 \citep{zhang2021gradient}. It holds that for any $\pi_\alpha^\prime$ and $\pi_\nu^\prime$:
    \begin{align*}
        &\Tilde{V}_{\rho}(\pi_\nu^\prime, \pi_\alpha) - \Tilde{V}_{\rho}(\pi_\nu, \pi_\alpha)\le \left\|\frac{\Tilde{d}_\rho^{\pi_{\nu}^\prime, \pi_\alpha}}{\Tilde{d}_\rho^{\pi_{\nu}, \pi_\alpha}}\right\|_\infty \max_{\Bar{\pi}_\nu}\langle\nabla_{\pi_\nu}\Tilde{V}_{\rho}(\pi_\nu, \pi_\alpha), \Bar{\pi}_\nu- \pi_\nu\rangle,\\
        &\Tilde{V}_{\rho}(\pi_\nu, \pi_\alpha) - V_{\rho}(\pi_\nu, \pi_\alpha^\prime)\le \left\|\frac{\Tilde{d}_\rho^{\pi_{\nu}, \pi_\alpha^\prime}}{\Tilde{d}_\rho^{\pi_{\nu}, \pi_\alpha}}\right\|_\infty \max_{\Bar{\pi}_\alpha}\langle\nabla_{\pi_\alpha}\Tilde{V}_{\rho}(\pi_\nu, \pi_\alpha), \pi_\alpha-\Bar{\pi}_\alpha\rangle.
    \end{align*}
        To prove Equation \eqref{eq:grad_alpha}, we denote $\pi_\alpha^\star\in \Tilde{\Pi}_\alpha^\star(\pi_\nu)$ to be the policy minimizing the quantity $\left\|\frac{\Tilde{d}_\rho^{\pi_\nu, \pi_\alpha}}{\rho}\right\|_{\infty}$. Then, we have
    \begin{align*}
        J_{\epsilon_\pi}(\nu, \alpha) -\min_{\alpha^\prime}J_{\epsilon_\pi}(\nu, \alpha^\prime)
        &=\Tilde{V}_\rho (\pi_\nu, {\pi}_\alpha) - \min_{\pi_\alpha^\prime}\Tilde{V}_{\rho}(\pi_\nu, {\pi}_\alpha^\prime)\\
        &= \Tilde{V}_\rho (\pi_\nu, {\pi}_\alpha) - \Tilde{V}_{\rho}(\pi_\nu, {\pi}_\alpha^\star)\\
        &\le \left\|\frac{\Tilde{d}_\rho^{\pi_{\nu}, \pi_\alpha^\star}}{\Tilde{d}_\rho^{\pi_{\nu}, \pi_\alpha}}\right\|_\infty \max_{\Bar{\pi}_\alpha}\langle\nabla_{\pi_\alpha}\Tilde{V}_{\rho}(\pi_\nu, \pi_\alpha), \Bar{\pi}_\alpha- \pi_\alpha\rangle\\
        &\le \frac{1}{1-\gamma}\left\|\frac{\Tilde{d}_\rho^{\pi_{\nu}, \pi_\alpha^\star}}{\rho}\right\|_\infty \max_{\Bar{\pi}_\alpha}\langle\nabla_{\pi_\alpha}\Tilde{V}_{\rho}(\pi_\nu, \pi_\alpha), \Bar{\pi}_\alpha- \pi_\alpha\rangle\\
        &\le \frac{1}{1-\gamma}C_{\Tilde{\cG}}\max_{\Bar{\pi}_\alpha}\langle\nabla_{\pi_\alpha}\Tilde{V}_{\rho}(\pi_\nu, \pi_\alpha), \Bar{\pi}_\alpha- \pi_\alpha\rangle.
    \end{align*}
    Note due to Equation \eqref{eq:V} and \eqref{eq:d}, we have $C_{\Tilde{\cG}} = C_{\cG}^{\epsilon_\pi}$, and $\nabla_{\pi_\alpha}\Tilde{V}_{\rho}(\pi_\nu, \pi_\alpha) = \nabla_{\pi_\alpha} J_{\epsilon_\pi}(\pi_\nu, \pi_\alpha) = \nabla_{\alpha} J_{\epsilon_\pi}(\nu,\alpha)$, proving Equation \eqref{eq:grad_alpha}.

    To prove Equation \eqref{eq:grad_nu}, we denote $\pi_\nu^\star\in \Tilde{\Pi}_\nu^\star(\pi_\alpha)$ to be the policy minimizing the quantity $\left\|\frac{\Tilde{d}_\rho^{\pi_\nu, \pi_\alpha}}{\rho}\right\|_{\infty}$. Therefore, we have that
    \begin{align*}
        \max_{\nu^\prime}J_{\epsilon_\pi}(\nu^\prime, \alpha) -J_{\epsilon_\pi}(\nu, \alpha)
        &=\max_{\pi_\nu^\prime}\Tilde{V}_\rho (\pi_\nu^\prime, {\pi}_\alpha) - \Tilde{V}_{\rho}(\pi_\nu, {\pi}_\alpha)\\
        &= \Tilde{V}_\rho (\pi_\nu^\star, {\pi}_\alpha) - \Tilde{V}_{\rho}(\pi_\nu, {\pi}_\alpha)\\
        &\le \left\|\frac{\Tilde{d}_\rho^{\pi_{\nu}^\star, \pi_\alpha}}{\Tilde{d}_\rho^{\pi_{\nu}, \pi_\alpha}}\right\|_\infty \max_{\Bar{\pi}_\nu}\langle\nabla_{\pi_\nu}\Tilde{V}_{\rho}(\pi_\nu, \pi_\alpha), \Bar{\pi}_\nu- \pi_\nu\rangle\\
        &\le \frac{1}{1-\gamma}\left\|\frac{\Tilde{d}_\rho^{\pi_{\nu}^\star, \pi_\alpha}}{\rho}\right\|_\infty \max_{\Bar{\pi}_\nu}\langle\nabla_{\pi_\nu}\Tilde{V}_{\rho}(\pi_\nu, \pi_\alpha), \Bar{\pi}_\nu- \pi_\nu\rangle\\
        &\le \frac{1}{1-\gamma}C_{\Tilde{\cG}}\max_{\Bar{\pi}_\nu}\langle\nabla_{\pi_\nu}\Tilde{V}_{\rho}(\pi_\nu, \pi_\alpha), \Bar{\pi}_\nu- \pi_\nu\rangle.
    \end{align*}
    Note due to Equation \eqref{eq:V} and \eqref{eq:d}, we have $C_{\Tilde{\cG}} = C_{\cG}^{\epsilon_\pi}$, and $\nabla_{\pi_\nu}\Tilde{V}_{\rho}(\pi_\nu, \pi_\alpha) = \nabla_{\pi_\nu} J_{\epsilon_\pi}(\pi_\nu, \pi_\alpha) = \nabla_{\nu} J_{\epsilon_\pi}(\nu,\alpha)$, proving Equation \eqref{eq:grad_nu}.
    \end{itemize}
\end{proof}

Before proving Theorem \ref{thm:max}, we need the following additional lemmas. Firstly, we define $\phi(\cdot):=\min_{\alpha}J_{\epsilon_\pi}(\cdot, \alpha)$, and the Moreau envelope for any $\lambda>0$ as $\phi_\lambda(\nu):=\max_{\nu^\prime} \phi(\nu^\prime) -\frac{1}{2\lambda}\|\nu-\nu^\prime\|^2$.
\begin{lemma}[Theorem 31 \citep{jin2020local}]\label{lemma:gd_max}
Suppose the function $J_{\epsilon_\pi}$ is $l$-smooth and $L$-Lipschitz. Then the output of Algorithm \ref{alg:max} with step size $\eta^t = \frac{1}{\sqrt{T+1}}$ satisfies
\begin{align*}
    \frac{1}{T}\sum_{t=1}^T\|\grad \phi_{1/2l}(\nu^t)\|^2\le 2\cdot \frac{\max_{\nu}\phi(\nu) -\phi(\nu^0)+ lL^2}{\sqrt{T+1}}.
\end{align*}
\end{lemma}
\begin{lemma}[Lemma 12 \citep{daskalakis2020independent}]\label{lemma:moreau}
    Suppose $J_{\epsilon_\pi}$ is $l$-smooth and $L$-Lipschitz. If there is some $u_\nu$ such that for any $\nu$, $\alpha$
    \begin{align*}
        \max_{\nu^\star}J_{\epsilon_\pi}(\nu^\star, \alpha) - J_{\epsilon_\pi}(\nu, \alpha)\le u_\nu \max_{\Bar{\nu}}\langle \nabla_{\nu} J_{\epsilon_\pi}(\nu, \alpha),\Bar{\nu}-\nu \rangle,
    \end{align*}
    then it holds that
    \begin{align*}
        \max_{\nu^\star}\phi(\nu^\star) - \phi(\nu)\le (u_\nu + \frac{L}{2l})\|\grad \phi_{1/2l}(\nu)\|^2. 
    \end{align*}
\end{lemma}

\begin{lemma}[Theorem 2a \citep{daskalakis2020independent}]
\label{lemma:tts_gda}
    Suppose $J_{\epsilon_\pi}$ is $l$-smooth and $L$-Lipschitz. If there is some $u_\nu$, $u_\alpha$ such that for any $\nu$, $\alpha$
    \begin{align*}
        &\max_{\nu^\star}J_{\epsilon_\pi}(\nu^\star, \alpha) - J_{\epsilon_\pi}(\nu, \alpha)\le u_\nu \max_{\Bar{\nu}}\langle \nabla_{\nu} J_{\epsilon_\pi}(\nu, \alpha),\Bar{\nu}-\nu \rangle,\\
        &J_{\epsilon_\pi}(\nu, \alpha) - \min_{\alpha^\star}J_{\epsilon_\pi}(\nu, \alpha^\star)\le u_\alpha \max_{\Bar{\alpha}}\langle \nabla_{\alpha} J_{\epsilon_\pi}(\nu, \alpha),\alpha-\Bar{\nu} \rangle.
    \end{align*}
    Then it holds that for $\eta_\alpha^t=\eta_\alpha=\Theta\left(\frac{{\delta}^4 u_\alpha^2}{l^3L^2(L / l+1)^2}\right)$, and $\eta_\nu^t=\eta_\nu=\Theta\left(\min\{\frac{\delta^8 u_\alpha^4}{l^5 (L / l+1)^4L^4}, \frac{\delta^2}{lL^2}\}\right)$, we have that
    \begin{align*}
        \frac{1}{T}\sum_{t=1}^T \max_{\nu^\star}\phi(\nu^\star)-\phi(\nu^t)\le \left({u_\nu} + \frac{L}{2l}\right)\delta,
    \end{align*}
    for $T \geq \Omega\left(\frac{\left(D_{\Pi_{\nu}}+D_{\Pi_{\alpha}}\right) L}{{\delta}^2 \eta_{\nu}}\right)$ iterations, where $D_{\Pi_{\nu}}$ and $D_{\Pi_{\alpha}}$ is the radius of the set $\Pi_\nu$, and $\Pi_{\alpha}$.
\end{lemma}

Now we are ready to prove Theorem \ref{thm:max}.
\begin{proof}
    We start from proving the convergence of Algorithm \ref{alg:max}. By substituting the parameters into Lemma \ref{lemma:gd_max}, we notice that the parameters satisfy that
    $L\le \frac{|\cA_\nu|+|\cA_\alpha|}{(1-\gamma)^2}$, $l\le \frac{2(|\cA_\nu|+|\cA_\alpha|)}{(1-\gamma)^3}$
     , $u_\nu = \frac{C_{\cG}^{\epsilon_\pi}}{1-\gamma}$ and $\phi(\nu)\le \frac{1}{1-\gamma}$.
    Therefore, we get
    \begin{align*}
        \frac{1}{T}\sum_{t=1}^T\|\grad \phi_{1/2l}(\nu^t)\|^2\le 2\cdot \frac{\max_{\nu}\phi(\nu) -\phi(\nu^0)+ lL^2}{\sqrt{T+1}}\le \frac{2}{(1-\gamma)\sqrt{T+1}} + \frac{4(|\cA_\nu| + |\cA_\alpha|)^2}{(1-\gamma)^7\sqrt{T+1}}.
    \end{align*}
    Now we use Lemma \ref{lemma:moreau}, and conclude that
    \begin{align*}
        \frac{1}{T}\sum_{t=1}^T \max_{\nu^\star}\phi(\nu^\star)-\phi(\nu^t)\le \frac{1}{\sqrt{T+1}}\operatorname{poly}(\frac{1}{1-\gamma}, |\cS|, |\cA_\nu|, |\cA_\alpha|, C_{\cG}^{\epsilon_\pi}).
    \end{align*}
    Therefore, for Algorithm \ref{alg:max} to achieve accuracy $\delta$, combined with the fact that $\phi(\cdot)=-\operatorname{Expl}(\cdot)$, one needs $T = \frac{1}{\delta^2}\operatorname{poly}(\frac{1}{1-\gamma}, |\cS|, |\cA_\alpha|, |\cA_\nu|, C_{\cG}^{\epsilon_\pi})$ total number of iterations, concluding the proof for Algorithm \ref{alg:max}.

The proof for Algorithm \ref{alg:tts} follows similar steps.
    We notice the instantiation of the parameters $L\le \frac{|\cA_\nu|+|\cA_\alpha|}{(1-\gamma)^2}$, $l\le \frac{2(|\cA_\nu|+|\cA_\alpha|)}{(1-\gamma)^3}$
     , $u_\nu , u_\alpha \le \frac{C_{\cG}^{\epsilon_\pi}}{1-\gamma}$, $D_{\Pi_{\nu}}$, $D_{\Pi_{\alpha}}\le \sqrt{|\cS|}$. Bt substituting the parameters into Lemma \ref{lemma:tts_gda} and combing it with the fact that $\phi(\cdot)=-\operatorname{Expl}(\cdot)$, we set total number of iterations $T = \operatorname{poly}(\frac{1}{\delta}, C_\cG^{\epsilon_\pi}, |\cS|, |\cA_\alpha|, |\cA_\nu|, \frac{1}{1-\gamma})$ to achieve accuracy $\delta$, concluding the proof for Algorithm \ref{alg:tts}.
\end{proof}

\newpage
\section{Experimental details}\label{app:exp}
\subsection{Introduction to the environment}
Kuhn Poker is a popular research game, which is extensive-form and zero-sum with discrete observation and action space. There exists an efficient min oracle with game tree search. For Robosumo competition, both agents are multi-leg robots and observe the position, velocity, and contact forces of joints in their body, and the position of their opponent’s joints, which is much more challenging due to the high dimensional observation and action space. 
\subsection{Feature importance in Robosumo environments}\label{app:feat_imp}
We show the important features for the Robosumo environment in Table \ref{tab:params}, \ref{tab:params1}, \ref{tab:params2}.
\begin{table}[hbt!]
\centering
\begin{tabular}{|c|c|c|c|} 
    \hline
    \multicolumn{1}{|c|}{Original} & \multicolumn{1}{|c|}{Local} & \multicolumn{1}{|c|}{Global} \\
    \hline
    38 & 32 & 31   \\
    \hline
    32 & 43 & 39   \\
    \hline
    33 & 38 & 29   \\
    \hline
    43 & 37 & 43   \\
    \hline
    41 & 31 & 41   \\
    \hline
    29 & 35 & 34   \\
    \hline
    35 & 42 & 42   \\
    \hline
    39 & 41 & 35   \\
    \hline
    42 & 40 & 37\\
    \hline
    26 & 28 & 44\\
     \hline
\end{tabular}
\caption{Feature Importance (Subset) in Robosumo Spider vs Spider.
"Original": Index of the most important features of the victim's observation while playing against normal opponents. 
"Local": Index of the most important features of the victim's observation while playing against the attacker trained via our constrained attack method.
"Global": Index of the most important features of the victim's observation while playing against the attacker trained via the unconstrained attack method as done in \citep{gleave2019adversarial}}.
\label{tab:params}
\end{table}

\begin{table}[hbt!]
\centering
\begin{tabular}{|c|c|c|c|} 
    \hline
    \multicolumn{1}{|c|}{Original} & \multicolumn{1}{|c|}{Local} & \multicolumn{1}{|c|}{Global} \\
    \hline
    25 & 25 & 23 \\
    \hline
    145  & 26 & 146 \\
    \hline
    26 & 145 &  145\\
    \hline
    146  & 24 & 25 \\
    \hline
    24  & 22 & 22 \\
    \hline
    142 & 146 & 21 \\
    \hline
    144  & 23  & 24\\
    \hline
    22  & 143 & 144 \\
    \hline
    143  & 27 & 142\\
    \hline
    23  & 142 & 143\\
     \hline
\end{tabular}
\caption{Feature Importance (Subset) in Robosumo Ant vs Ant.
"Original": Index of the most important features of the victim's observation while playing against normal opponents. 
"Local": Index of the most important features of the victim's observation while playing against the attacker trained via our constrained attack method.
"Global": Index of the most important features of the victim's observation while playing against the attacker trained via the unconstrained attack method as done in \citep{gleave2019adversarial}}.
\label{tab:params1}
\end{table}

\begin{table}[hbt!]
\centering
\begin{tabular}{|c|c|c|c|} 
    \hline
    \multicolumn{1}{|c|}{Original} & \multicolumn{1}{|c|}{Local} & \multicolumn{1}{|c|}{Global} \\
    \hline
    193 & 27 & 33   \\
    \hline
    197 & 35 & 25   \\
    \hline
    25 & 25 & 29   \\
    \hline
    35 & 197 & 27   \\
    \hline
    27 & 193 & 35   \\
    \hline
    29 & 189 & 31   \\
    \hline
    191 & 199 & 36   \\
    \hline
    199 & 34 & 34   \\
    \hline
    190 & 191 & 28\\
    \hline
    36 & 36 & 26\\
     \hline
\end{tabular}
\caption{Feature Importance (Subset) in Robosumo Bug vs Bug.
"Original": Index of the most important features of the victim's observation while playing against normal opponents. 
"Local": Index of the most important features of the victim's observation while playing against the attacker trained via our constrained attack method.
"Global": Index of the most important features of the victim's observation while playing against the attacker trained via the unconstrained attack method as done in \citep{gleave2019adversarial}}.
\label{tab:params2}
\end{table}
\subsection{Implementation details.}
\paragraph{Kuhn Poker.} For implementing our adversarial training approach, we adopt both advanced A2C and RPG policy gradient methods. The policy is parameterized by an MLP with a hidden layer size of $128$. We use a batch size of $16$ and SGD as the optimizer.

\paragraph{Robosumo competition.} We use PPO as our base policy optimization algorithm. We use a clip range of $0.2$ and $0.95$ for GAE. The number of hidden units for parameterization of the policy is $64$. For all adversarial training algorithms with both a single timescale and two timescales and the baseline algorithm self-play and fictitious-play, we train by $3e7$ timesteps with a batch size of $64$. For attack, we train for $2e7$ timesteps with a batch size of $2048$. We use Adam as the optimizer.
\subsection{Additional experimental results}\label{addi_res}
\paragraph{Trade-off between stealthiness and winning rate.}We have mentioned that there is a trade-off between stealthiness and the winning rate. To quantitatively show the trade-off, we show the comparison of our constrained attack and unconstrained one in Figure \ref{fig:local_attack}, where ${\epsilon_\pi}=\lambda$ in the figures.
\begin{figure*}[!htbp]
\centering
\includegraphics[width=0.9\textwidth]{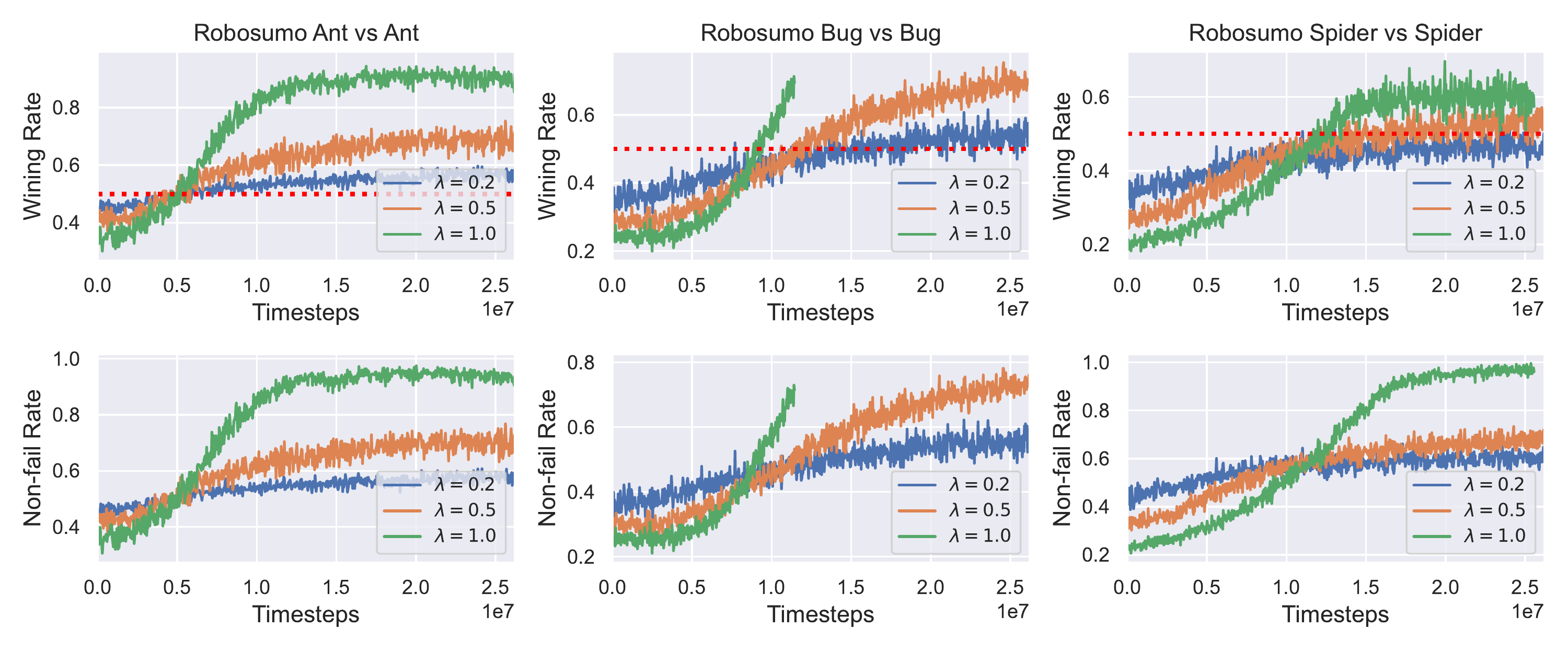}
\vspace{-2mm}
\caption
{Winning rate and Non-fail rate of different attack methods. }
\label{fig:local_attack}
\end{figure*}

\paragraph{Additional results on the state-distribution shifts.}To quantify the state-distribution shift induced due to the unconstrained (global) and constraint (local) attacks, in addition to the Wasserstein distance in Figure \ref{fig:wasserstein_plot}, we also compute the KL-Divergence and Hellinger distance between the state-distribution of the unattacked policy with the global and local attacks respectively. We observe that the state-distribution shift incurred is much lesser in local attacks than global, as also can be seen in Figure \ref{fig:feat_imp},\ref{fig:wasserstein_plot}, highlighting the stealthiness of our attack.

\begin{figure*}[!t]
  \centering
\begin{subfigure}{\textwidth}
\begin{center}
\includegraphics[width=.5\textwidth]{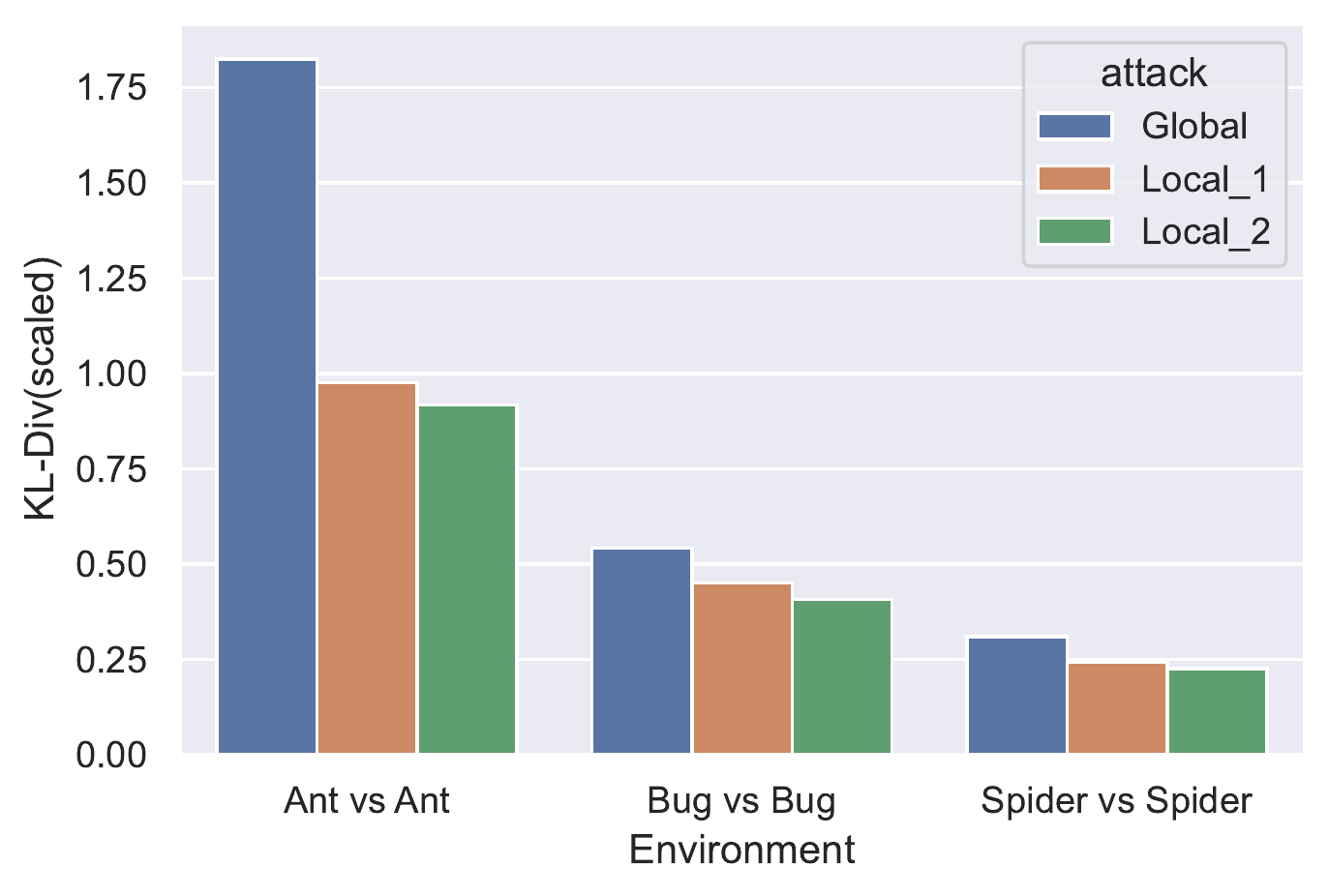}
\end{center}
\vspace{-1mm}
\caption{KL-Divergence(Scaled)}
\end{subfigure}
\begin{subfigure}{\textwidth}
\begin{center}
\includegraphics[width=.5\textwidth]{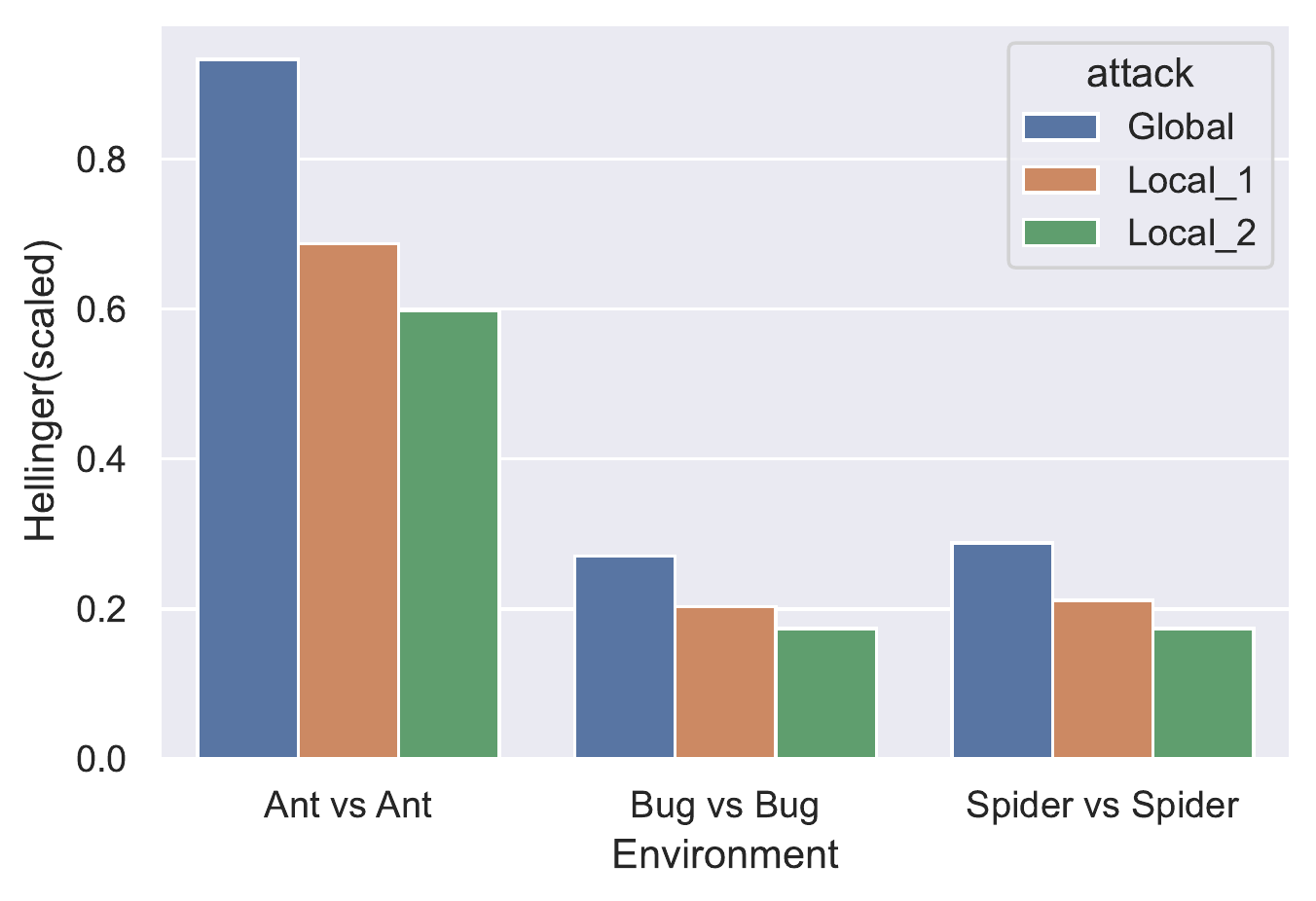}
\end{center}
\vspace{-2mm}
\caption{Hellinger Distance(Scaled)}
\end{subfigure}
\hfill
\caption{This figure compares the state-distribution shift w.r.t KL-Divergence(Scaled) and Hellinger Distance(scaled) incurred due to the Global (${\epsilon_\pi} = 1$), Local-1 (${\epsilon_\pi} = 0.7$),  Local-2(${\epsilon_\pi} = 0.3$) attacks in $3$ Robosumo environments. The plot clearly demonstrates the Stealthiness of constraint (local) attacks in preserving the state-distribution shifts and achieving stealthiness.}
\label{add_div}
\end{figure*}

\paragraph{Mismatch of attack budgets between training time and test time.} We show the attacker’s performance when there is a mismatch between the actual attack budget and the budget used to train the robust victim in the following table. It is clear that \textbf{even if the defense budget is not correctly specified, the victim policy still shows greatly improved robustness}, although a bit worse than the case when the defense budget is correctly specified.

\begin{table}[!ht]
\center
\begin{tabular}{|l|l|l|l|}
\hline
Attacker score       & Attack budget $0.3$ & Attack budget $0.7$ & Attack budget $1.0$ \\ \hline
Defense budget $0.3$ & \textbf{0.38}       & 0.50                & 0.71                \\ \hline
Defense budget $0.7$ & 0.41                & \textbf{0.43}       & 0.53                \\ \hline
Defense budget $1.0$ & 0.44                & 0.45                & \textbf{0.50}       \\ \hline
No defense           & 0.59                & 0.78                & 0.90                \\ \hline
\end{tabular}
\caption{Attacker’s score with different attack budgets against robust victim policies with different defense budgets and without defenses}
\end{table}

\end{document}